%% file: iclr2026_conference.tex
\documentclass{article} 
\usepackage{iclr2026_conference,times}

\input{math_commands.tex}

\usepackage{hyperref}
\usepackage{url}
\usepackage{graphicx}
\usepackage{amsmath,amsthm,amssymb}
\usepackage{thmtools, thm-restate}
\usepackage{wrapfig}
\usepackage{algorithm}
\usepackage{algpseudocode}
\usepackage{dsfont}
\usepackage{enumitem}
\usepackage{booktabs}
\usepackage{multirow}
\usepackage{array}
\usepackage[table]{xcolor}
\usepackage{amsmath}      
\usepackage{amsthm}       
\usepackage{xcolor}       
\usepackage[most]{tcolorbox} 

\newcommand{\rev}[1]{#1}
\newcommand{\tea}[0]{p_{\text{E}}}
\newcommand{\stu}[0]{p_{\text{T}}}

\title{Trust-Region Adaptive Policy Optimization}

\author{Mingyu Su$^{1}$, Jian Guan$^{2}$, Yuxian Gu$^{1}$, Minlie Huang$^{1}$, Hongning Wang$^{1}\thanks{Corresponding Author}$ \\
$^1$The Conversational AI (CoAI) Group, Tsinghua University $^2$Ant Group\\
\texttt{sumy25@mails.tsinghua.edu.cn, hw-ai@tsinghua.edu.cn}
}

\def\D{\mathcal{D}}
\newcommand{\mySFT}{TrSFT}
\newcommand{\mySolution}{TRAPO}

\iclrfinalcopy 
\begin{document}

\maketitle

\begin{abstract}
Post-training methods, especially Supervised Fine-Tuning (SFT) and Reinforcement Learning (RL), play an important role in improving large language models' (LLMs) complex reasoning abilities. However, the dominant two-stage pipeline (SFT then RL) suffers from a key inconsistency: SFT enforces rigid imitation that suppresses exploration and induces forgetting, limiting RL's potential for improvements. 
We address this inefficiency with \mySolution{}  (\textbf{T}rust-\textbf{R}egion \textbf{A}daptive \textbf{P}olicy \textbf{O}ptimization), a hybrid framework that interleaves SFT and RL within each training instance by optimizing SFT loss on expert prefixes and RL loss on the model's own completions, unifying external supervision and self-exploration. To stabilize training, we introduce Trust-Region SFT (TrSFT), which minimizes forward KL divergence inside a trust region but attenuates optimization outside, effectively shifting toward reverse KL and yielding stable, mode-seeking updates favorable for RL. An adaptive prefix-selection mechanism further allocates expert guidance based on measured utility. Experiments on five mathematical reasoning benchmarks show that \mySolution\ consistently surpasses standard SFT, RL, and SFT-then-RL pipelines, as well as recent state-of-the-art approaches, establishing a strong new paradigm for reasoning-enhanced LLMs.
Our code and data are publicly available at \url{https://github.com/Su-my/TRAPO}.
\end{abstract}
\input{chapters/introduction}

\input{chapters/methodology_version2}
\input{chapters/experiments}
\input{chapters/related_works}

\section{Conclusion}
In this work, we propose \mySolution, a one-stage training paradigm that unifies SFT and RL under expert prefix guidance. \mySolution\ enables both the internalization of guidance and its use in generating high-quality continuations. To stabilize training, we introduce \mySFT, which constrains policy updates within a trust region, mitigating the disruptive effect of low-probability tokens and effectively shifting Forward KL’s mode-covering behavior toward Reverse KL’s mode-seeking behavior. \mySFT\ integrates seamlessly into RL to enhance guidance utilization. We further design micro-group sampling, which adaptively adjusts guidance length based on return improvements, balancing exploration with expert supervision. Experiments on mathematical reasoning tasks show that \mySolution\ significantly outperforms standalone SFT, RL, and the conventional SFT-then-RL pipeline.

\bibliography{iclr2026_conference}
\bibliographystyle{iclr2026_conference}

\newpage
\input{chapters/appendix}
\end{document}

%% file: math_commands.tex
\usepackage{amsmath,amsfonts,bm}

\def\1{\bm{1}}

\DeclareMathAlphabet{\mathsfit}{\encodingdefault}{\sfdefault}{m}{sl}
\SetMathAlphabet{\mathsfit}{bold}{\encodingdefault}{\sfdefault}{bx}{n}



%% file: chapters/introduction.tex
\section{Introduction}


Driven largely by post-training, especially Supervised Fine-Tuning (SFT) and Reinforcement Learning (RL) techniques, large language models (LLMs) have demonstrated substantial advances in complex reasoning tasks, such as mathematical deduction~\citep{lightman2023let}, program synthesis~\citep{chen2021evaluatinglargelanguagemodels}, and multi-step decision-making~\citep{shridhar2021alfworld}, shifting them from surface-level chat-bots toward deep reasoners~\citep{openAI-o3,guo2025deepseek}. 
Particularly, 
current dominant LLM post-training pipeline integrates SFT and RL in a two-stage fashion, where SFT precedes RL. This design counts on the SFT stage to teach an LLM imitating carefully curated expert demonstrations, and then sharpens the model's reasoning skills via trial-and-error in the RL stage. 

However, a fundamental barrier exists that hinders the synergy between SFT and RL.
On the one hand, SFT tends to lock the trained models into imitative and rigid behavior modes~\citep{chu2025sft, chen2025sft}, which hinders effective exploration crucially needed in the RL stage. 
On the other hand, SFT is also prone to cause catastrophic forgetting in the trained models, impeding the RL stage from exploiting the pretraining knowledge for improvements.
These inconsistencies raise an important challenge: \textit{how can we effectively incorporate the knowledge-distillation benefits of SFT into RL training without undermining the model's exploratory capacity and pretrained knowledge?} 

In this work, we pinpoint the root cause of the aforementioned barrier to be the disjoint two-stage SFT and RL training, and propose a hybrid post-training framework as our solution.
Specifically, standard SFT training requires ``sufficient'' reduction of its training loss, but this is not sufficiently an indicative metric for later RL training. In other words, a lower SFT loss does not suggest a better starting point for RL training. Worse still, over-training in SFT might push the model out of the region suited for RL, while there is no immediate signal measuring that. 
Our idea is then to interleave SFT with RL in each training instance: only perform SFT on the prefix of the given expert trajectory, and continue RL training by completing the trajectory thereafter.
In this way, we can best integrate the utility of expert demonstrations and self-exploration at the finest grain.


\begin{figure}[t!]
\begin{center}
\includegraphics[width=\linewidth]{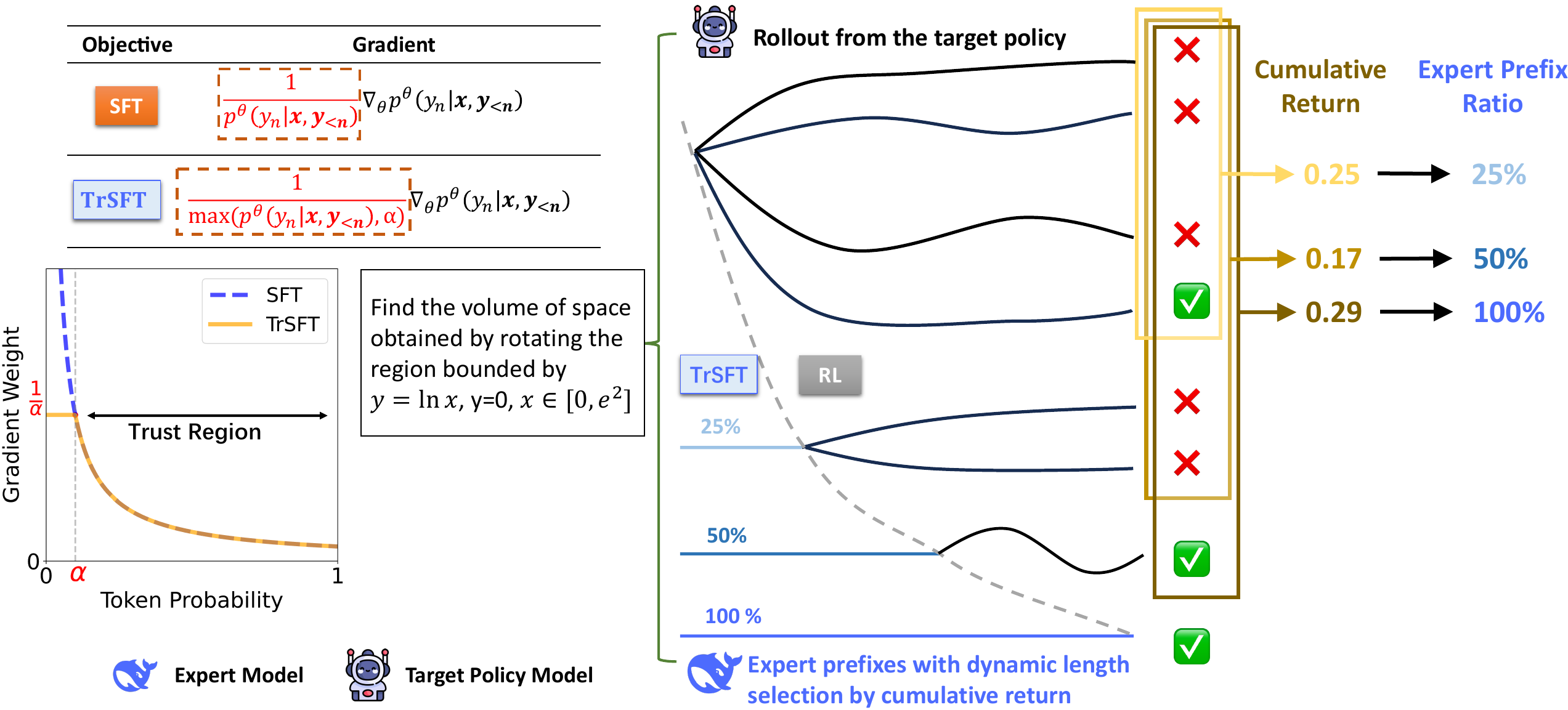}
\end{center}
\vspace{-1.0em}

\caption{
An overview of our \textbf{\mySolution} framework, which synergistically combines two key components: Trust-Region SFT (TrSFT) and Adaptive Expert Guidance.
\textbf{Left:} By clipping the gradient weight with a trust-region parameter $\alpha$, \mySFT\ prevents exploding gradients on low-probability tokens, ensuring a stable learning signal when combining with RL.
\textbf{Right:} The adaptive guidance mechanism implements a ``learn-while-practicing'' loop. When the target model fails a rollout, its cumulative return dynamically dictates the length of an expert prefix provided for guidance. The model then continues generation and the full trajectory is optimized using both the \mySFT\ loss on expert prefixes and a standard RL objective on the model rollout.}

\label{fig:overview}
\end{figure}

Though intuitive, several technical challenges prevent a straightforward combining scheme.
First, using the language of RL, the negative log-likelihood (NLL) objective in SFT aims to minimize the forward Kullback–Leibler (KL) divergence between the target and expert policies (i.e., behavior cloning~\citep{torabi2018behavioral}), which exhibits a strong mode-covering property by assigning relatively high probabilities even to regions where the expert policy has no support~\citep{minka2005divergence, malinin2019reverse}. 
This is especially toxic when one interleaves SFT with RL in a per-instance basis, as those inflated modes immediately cause degenerations in the target policy (e.g., repetition or erroneous decoding) and keeps RL away from effective exploration. 
In this work, to mitigate the negative impact of SFT on RL, we propose a new form of SFT, named Trust-Region SFT (\mySFT), which minimizes forward KL divergence inside a trust region, while outside the region it down-weights the optimization strength and effectively shifts the objective toward reverse KL.
As reverse KL minimization is characterized by mode-seeking behaviors~\citep{gu2023minillm}, TrSFT avoids blindly balancing the modes in the entire space, and only seeks the most prominent mode of the expert policy when the expert policy is clearly distinct from the current target policy. This provides a stable and robust policy for the following RL training.



Second, determining the optimal amount of expert guidance for each training instance is nontrivial. A ``one-size-fits-all'' prefix length is inherently inefficient: it stifles valuable exploration on problems the model could solve independently, while offering insufficient support on challenging problems where more guidance is crucial. To address this, we introduce an adaptive prefix-selection mechanism that dynamically allocates expert guidance based on its utility. The core idea is a scaffolding approach: the model first attempts a problem without any guidance. If it struggles, indicated by a low return, we incrementally provide longer expert prefixes in subsequent rollouts for that same problem. This strategy ensures that expert guidance is used judiciously only as much as needed, thereby creating a dynamic balance between self-exploration and expert guidance.

Building on these designs, we propose \textbf{\mySolution} (\textbf{T}rust-\textbf{R}egion \textbf{A}daptive \textbf{P}olicy \textbf{O}ptimization), which unifies SFT and RL in a principled way. Extensive experiments on five mathematical reasoning benchmarks demonstrate the substantial benefits of our approach. \mySolution\ delivers significant gains over conventional post-training algorithms. Specifically, it outperforms standalone SFT and pure RL by +6.3 and +6.2 points, respectively. More critically, \mySolution\ surpasses the highly competitive SFT-then-RL baseline by +2.3 points, proving the superiority of our integrated framework over the disjoint two-stage approach. Crucially, we show that a naive, direct combination of SFT and RL objectives is counterproductive, causing a catastrophic performance collapse of over 18 points compared to the pure RL baseline. This result validates our initial hypothesis and highlights the necessity of our TrSFT objective for achieving a stable and synergistic integration. Furthermore, our pass@k analysis reveals that \mySolution, unlike pure RL methods, effectively expands the model's underlying solution space, leading to stronger test-time scaling properties. These compelling results establish \mySolution\ as a new, effective paradigm for developing reasoning-enhanced LLMs.


In summary, this work makes several pivotal contributions:

I. We introduce \mySolution, 
a novel post-training framework that combines SFT and RL at the instance level. It features TrSFT for stable knowledge internalization and a dynamic guidance-selection mechanism to balance guidance and exploration.


II. We identify the mode-covering property of SFT (Forward KL) as a source of instability and theoretically show that TrSFT shifts the optimization objective from SFT's mode-covering towards Reverse KL's mode-seeking behavior, ensuring stable updates for RL.

III. We conduct extensive experiments on five mathematical reasoning benchmarks, demonstrating that \mySolution\ outperforms traditional SFT, RL, and SFT-then-RL pipelines, as well as recent state-of-the-art approaches that combine SFT and RL.



%% file: chapters/methodology_version2.tex
\section{Methodology}
\label{methodology}


Our proposed framework, \mySolution{}, guides RL by leveraging prefixes from offline expert trajectories (Figure~\ref{fig:prefix-effect}). These prefixes serve a dual purpose: they act as in-context demonstrations to guide exploration and as direct supervision signals for skill internalization. The successful implementation of \mySolution\ hinges on resolving two key challenges: \textbf{(1) Guidance internalization:} how to effectively learn from the prefixes, and \textbf{(2) Guidance selection:} how to choose the optimal prefix length for each prompt. We will now provide an overview of the framework before detailing each component.

\subsection{Overview of \mySolution}
\begin{wrapfigure}{r}{0.4\linewidth}
    \vspace{-5.2em}
    \centering
    \includegraphics[width=\linewidth]{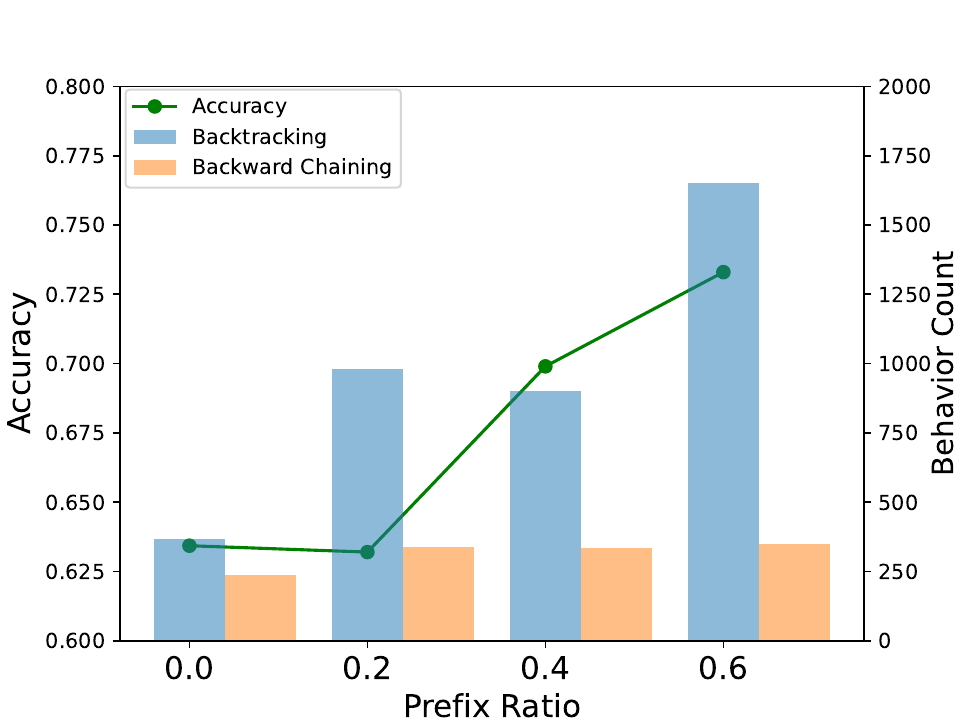}
    \vspace{-1.5em}
    \caption{Accuracy and characteristic of Qwen2.5-3B-Instruct reasoning on MATH-500 with different amount of tokens from DeepSeek-R1 as prefixes.}
    \label{fig:prefix-effect}
    \vspace{-10pt}
\end{wrapfigure}

The design of \mySolution{} is to realize a synergistic ``learn-while-practicing'' paradigm for LLM post-training, which couples learning from offline expert trajectories with online RL updates. 
The core workflow unfolds as follows: for each prompt, (1) a selected prefix from an offline expert trajectory is provided as a starting context; (2) the target policy rolls out from there to complete the reasoning; and (3) a dual update is then performed, where the generated completion is used for standard RL updates, while the expert prefix is used for direct policy optimization to internalize the expert's reasoning skills. 
\mySolution{} is expected to guide the target policy to find high return completions via the selected expert prefixes~\citep{liu2025uft, huang2025blending} and pruning unproductive solution space via skill internalization from the expert. 
As shown in Figure~\ref{fig:prefix-effect}, we empirically validate the benefit of introducing prefixes from expert trajectories in target policy rollouts. On the MATH-500 benchmark, we provided Qwen2.5-3B-Instruct with prefixes from DeepSeek-R1, and then calculated the response accuracy and counted the frequency of two key reasoning behaviors, i.e., backtracking and backward chaining~\citep{gandhi2025cognitive}, in the completed suffixes. Detailed experimental procedures are provided in Appendix~\ref{appendix:pre-exp}. It is clearly observed that longer expert prefixes steadily improve accuracy and stimulate the emergence of advanced reasoning behaviors.


However, integrating RL with expert demonstration presents two significant challenges:  \textbf{(1) Effective learning objective for guidance internalization.} The aim of incorporating expert prefixes is to not only immediately guide the target policy to find high return trajectories, but also enable the target policy to distill the expert's core problem-solving skills, highlighting the need for an effective joint learning algorithm. \textbf{(2) Efficient and adaptive guidance selection.} When using expert prefixes to guide exploration, a ``one-size-fits-all'' strategy, which applies a uniform prefix selection and a fixed number of rollouts to all prompts~\citep{liu2025uft,huang2025blending}, can hardly be optimal. This simple method would over-guide in easy prompts, thus stifling valuable exploration, while under-guiding in difficult ones, leading to failed rollouts. It remains challenging to design an adaptive, prompt specific mechanism that allocates only the minimal guidance necessary to ensure the target policy can effectively discover the correct solution paths.

\subsection{Trust-region SFT }

Regarding the first challenge of internalizing an expert's reasoning skills, a straightforward approach is to combine the standard SFT loss on expert prefixes with an RL objective on model-generated rollouts. However, this naive combination leads to severe performance degradation in our experiments (see \S\ref{sec:ablation}). To uncover the root cause, we first conduct a pilot study of SFT's training dynamics. The findings motivate our proposed Trust-Region SFT (\mySFT), an objective designed to resolve this instability, whose theoretical properties we then analyze.

\paragraph{Training Dynamics of SFT}

We begin by examining the standard SFT training objective. Given an expert policy $\tea$ and target policy $\stu^{\theta}$ parameterized by $\theta$, SFT forces $\stu^{\theta}$ to mimic $\tea$ by minimizing the NLL in every trajectory $\boldsymbol{y}=(y_1, \cdots,y_n)$ from the expert policy in a given prompt set $X$,
\begin{equation}
    L_{\text{SFT}}(\theta)=\mathbb{E}_{\boldsymbol{x}\sim X}\left[\mathbb{E}_{\boldsymbol{y}\sim \tea(\cdot|\boldsymbol{x})}\left[- \log\stu^\theta(\boldsymbol{y}|\boldsymbol{x})\right]\right].\label{eq:SFT-objective}
\end{equation}

\begin{figure}[t!]
\begin{center}
\includegraphics[width=\linewidth]{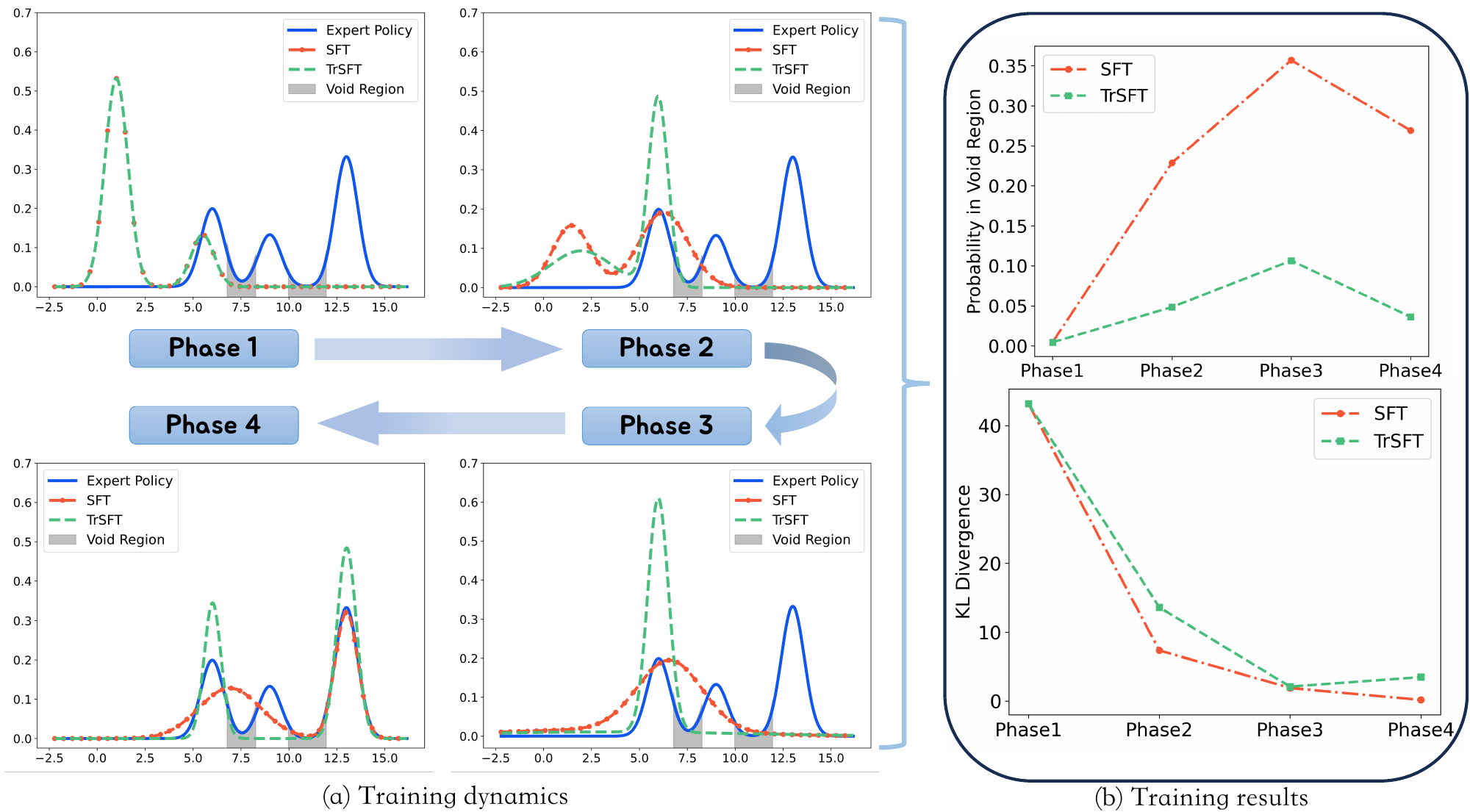}
\end{center}
\vspace{-1.5em}
\caption{An illustrative experiment showing the training dynamics during SFT.
Panel (a) shows snapshots of learnt target policy at four consecutive training phases, corresponding to training steps of 0, 50, 100, and 1000, respectively. Panel (b) presents the KL divergence curve along with the change in cumulative probability of the target policy within the void regions.
}
\label{fig:toy-experiment}
\end{figure}


The SFT objective is equivalent to minimizing the cumulative token-level forward KL divergence (see Appendix \ref{appendix:sft-kl-eq} for derivation). To understand its training dynamics, we conduct an illustrative experiment training a two-mode Gaussian Mixture Model (GMM) to mimic a three-mode expert GMM. As shown in Figure~\ref{fig:toy-experiment}, this process reveals the distribution-blending phenomenon \citep{minka2005divergence, malinin2019reverse}: the target policy assigns probability to void regions unsupported by either policy (e.g., shaded regions in Figure~\ref{fig:toy-experiment} (a)). This effect is detrimental to RL, as these regions lead to degenerated outputs (e.g., repetitions) that hinder effective exploration.

We point out that the underlying cause of this phenomenon lies in the gradient of standard SFT loss:
\begin{equation}
    \nabla_{\theta} L_{\text{SFT}} = -\frac{1}{N} \sum_{i=1}^{N} \sum_{n=1}^{|\boldsymbol{y^i}|} \textcolor{red}{\frac{1}{\stu^{\theta}(y_n^i|\boldsymbol{x^i},\boldsymbol{y_{<n}^i})}}\nabla_{\theta}\stu^{\theta}(y_n^i|\boldsymbol{x^i},\boldsymbol{y_{<n}^i}). \label{eq:SFT-gradient}
\end{equation}

The term $\frac{1}{\stu^{\theta}(y_n^i|\boldsymbol{x^i},\boldsymbol{y_{<n}^i})}$ in Eq \eqref{eq:SFT-gradient} can be considered as the weight of training token $y_{n}^i$ given 
context $(\boldsymbol{x^i},\boldsymbol{y_{<n}^i})$. When $y_{n}^i$ is drawn from a mode of expert policy which is far away from the current policy's mode, e.g., the rightmost mode of expert policy in Figure \ref{fig:toy-experiment}, this weight term will be extremely large. Consequently, the resulting $\nabla_{\theta} L_{\text{SFT}}$ is more detrimental than helpful for $\stu^\theta$ update, though it points to the correct the direction of policy improvement: the inflated gradient tends to first push $\stu^\theta$ into the voided regions during the online gradient update of SFT, as shown in Figure \ref{fig:toy-experiment}. 


While sufficient training can eventually correct for distribution-blending in standard SFT (e.g., phase 4 in Figure \ref{fig:toy-experiment}), this issue is critical in \mySolution{}. Here, the interleaving of SFT and RL means that any probability mass allocated to void regions will immediately cause degenerated rollouts. This can significantly hinder model learning, making modifications to the standard SFT objective essential.

\paragraph{Clipping the gradient of standard SFT loss}
To mitigate the negative impact of distribution-blending during online model update, we propose Trust-Region SFT (TrSFT), a mechanism designed to adaptively leverage the SFT updates. 
The core idea is to establish a region where the gradient from standard SFT loss can be trusted, while intervention is needed outside this region to prevent model collapsing into undesirable parts of the solution space:
\begin{equation}
    \nabla_{\theta} L_{\text{\mySFT}}^\alpha = -\frac{1}{N} \sum_{i=1}^{N} \sum_{n=1}^{|\boldsymbol{y^i}|} \textcolor{red}{\frac{1}{\max\Big(\stu^{\theta}(y_n^i|\boldsymbol{x^i},\boldsymbol{y_{<n}^i}), \alpha\Big)}}\nabla_{\theta}\stu^{\theta}(y_n^i|\boldsymbol{x^i},\boldsymbol{y_{<n}^i}), \label{eq:TrSFT-gradient}
\end{equation}
where $\alpha\in[0,1]$ is a hyperparameter defining the boundary of trust region. 
The optimization objective exhibits several desirable properties:



\textbf{(1) Safe knowledge instillation within the trust region.} Standard SFT risks distribution blending by applying large updates towards ``distant'' modes of the expert policy. \mySFT\ mitigates this by defining a dynamic trust region based on the target policy's own beliefs (i.e., whether $\stu^\theta(y_n^i|\boldsymbol{x^i},\boldsymbol{y_{<n}^i}) \ge \alpha$). Within this region, it employs the standard SFT objective
to aggressively mimic expert policy's behavior. Outside the region, the constant weight $\frac{1}{\alpha}$ significantly dampens the gradients, thereby reducing the disruptive impact of large gradient updates on the target policy's immediate behaviors.

The benefit of \mySFT\ can be easily understood by our illustrative experiment in Figure~\ref{fig:toy-experiment}. 
As illustrated in Phases 2 and 3, TrSFT prioritizes learning from high-overlapping regions---modes jointly favored by $\tea$ and $\stu^{\theta}$. This helps the target policy first instills expert's wisdom that is ``close'' and readily tangible, while preserving its existing strengths. This conservative, trust-region-based approach is safer and more stable than aggressive, full-distribution matching, especially when the target policy's capacity is limited.

\textbf{(2) Leading to a reasonable optimization endpoint.} 
We present the optimization problem defined by the gradient defined in Eq.~\eqref{eq:TrSFT-gradient}, and theoretically derive the solution.
\begin{restatable}{prop}{Firstprop}
Let $S(\lambda) = \{\, c \mid \tea(c) > \alpha \lambda, c\in\mathcal{C} \}$ and $\mathcal{C}$ is the vocabulary. There exists a \emph{unique} $\lambda \in \big(0, 1\big)$ such that $\lambda = \sum_{c \in S(\lambda)} \tea(c)\;$. And for this $\lambda$, the optimal solution of the optimization problem defined by the gradient defined in Eq.~\eqref{eq:TrSFT-gradient} is given by
\[
\stu^*({c}) =
\begin{cases}
\dfrac{\tea(c)}{\lambda}, & \text{if } \tea(c) > \alpha \lambda. \\[8pt]
0, & \text{otherwise }
\end{cases}
\] \label{prop:solution}
\end{restatable}
The optimal solution for \mySFT\ counters distribution-blending by pruning low-probability regions in expert policy ($\stu^*(c) = 0$) and rescaling primary modes ($\stu^*(c) = \tea(c) / \lambda$). This dual action effectively transforms the objective from mode-covering of Forward KL to mode-seeking akin to reverse KL, forcing the policy to focus on the expert's core skills and thereby facilitating high-return rollouts for RL. Appendix~\ref{appendix:solution} shows more details.

\subsection{Micro-group sampling}
To elegantly address the second challenge of guidance selection, we propose \emph{micro-group sampling}, which adaptively allocates guidance from expert prefixes based on the observed returns from the current policy rollouts, thereby both minimizing unnecessary reliance on expert prefixes and accommodating the heterogeneity of prompt difficulty within each training batch.

As illustrated in Figure \ref{fig:overview}, in each training prompt, \mySolution\ creates $N$ micro-groups in order, where each micro-group $g_i$ (for $i=1, \dots, N$) is specified by three key hyper-parameters: the prefix length ratio $L_i$, the return threshold $t_i$, and the sampling budget $n_i$. 
For micro-group $g_i$, \mySolution{} first computes the average return from all samples generated in the preceding micro-groups. If the average return is smaller than the threshold $t_i$, \mySolution{} provides the current target policy with a prefix whose length is set by the ratio $L_i$ of the complete expert trajectory, and then samples $n_i$ completions from the target policy. 
Otherwise, no expert prefix is provided, and $n_i$ policy rollouts are obtained directly from the target policy.

We set $0=L_1 < L_2 < \cdots < L_N=1$. 
$L_1 = 0$ is to ensure that in each training prompt, \mySolution{} always starts from guidance-free self-exploration RL, and $L_N = 1$ allows the target policy, when necessary, to access the complete reasoning path from the expert. 
As a result, an increasing level of $L_i$ ensures richer guidance is provided only when shorter prefixes prove insufficient. 
For clarity, Appendix~\ref{appendix:pseudocode} details the complete training procedure specified by \mySolution{} in Algorithm~\ref{alg:tag}.

%% file: chapters/experiments.tex
\section{experiments}
\subsection{Experimental Setup}
\label{sec:setup}
\paragraph{Training Details}
Our primary training dataset is OpenR1-Math-46k-8192~\citep{yan2025learning}, which consists of a large collection of verified reasoning trajectories generated by DeepSeek-R1 for complex mathematical problems. To enhance the diversity of guidance, we additionally pair each problem with another trajectory sampled from OpenR1-Math-200k~\citep{openr1}. 
Following recent work~\citep{yan2025learning, huang2025blending, fu2025srft}, we use Qwen2.5-Math-7B~\citep{yang2024qwen2} as the base model.  To further validate the generality of our approach, we additionally evaluate it on Qwen2.5-7B-Instruct~\citep{qwen2.5}, a general-purpose model. 

\paragraph{Implementation Details}

We adopt the Group Relative Policy Optimization (GRPO)~\citep{shao2024deepseekmath,liu2025understanding} algorithm without the KL penalty \citep{hu2025open} for RL. 
The training is configured with a batch size of 128 and a constant learning rate of $5 \times 10^{-6}$. Our adaptive guidance mechanism operates on a total group size of 8 following prior work, which is partition into four micro-groups of sizes \{4, 2, 1, 1\}. These micro-groups correspond to relative expert prefix-length proportions of $(L_1, \dots, L_4) = (0, 0.2, 0.5, 1.0)$ and are activated by reward thresholds of $(t_1, \dots, t_4) = (-1, 0.5, 0.7, 0.9)$, respectively. The threshold $t_1=-1$ ensures the first micro-group is always without guidance. For our TrSFT objective, the trust-region parameter $\alpha$ in Eq.~\eqref{eq:TrSFT-gradient} is set to 0.1.
More details can be found in Appendix~\ref{appendix:main_exp}.


\paragraph{Evaluation benchmark and metrics}
We focus on mathematical reasoning tasks, while also evaluating various methods on both mathematical and general-domain benchmarks. Specifically, the mathematical benchmarks include AIME2024~\citep{li2024numinamath}, AMC~\citep{DBLP:conf/acl/HeLBHTSHHHZLQL024}, Minerva~\citep{DBLP:conf/nips/LewkowyczADDMRS22}, OlympiadBench~\citep{DBLP:conf/acl/HeLBHTSHHHZLQL024}, and MATH-500~\citep{DBLP:conf/nips/HendrycksBKABTS21}. Given the relatively small number of test samples in AIME2024 and AMC, we report avg@32 on these benchmarks, while employing pass@1 for the remaining three. For general-domain reasoning benchmarks, we report pass@1 on ARC-c~\citep{DBLP:journals/corr/abs-1803-05457}
and MMLU-Pro~\citep{DBLP:conf/nips/WangMZNCGRAHJLK24} to examine whether the improvements in reasoning ability generalize to other reasoning tasks.

\paragraph{Baselines}
We consider two categories of baseline methods:
\begin{itemize}[leftmargin=10pt]
    \item \textbf{Pure RL without External Expert Guidance.} This category includes GRPO~\citep{shao2024deepseekmath}, PRIME-Zero~\citep{cui2025process}, SimpleRL-Zero~\citep{zeng2025simplerl}, OpenReasoner-Zero~\citep{hu2025open}, and Oat-Zero~\citep{liu2025understanding}.
    \item \textbf{RL Incorporating External Expert Guidance.} This category includes (1) SFT: directly training the model to imitate expert trajectories, (2) SFT-then-RL: following the standard two-stage pipeline where SFT precedes RL, 
    (3) LUFFY~\citep{yan2025learning}: augmenting each group of eight rollouts with one expert trajectory to perform offline RL, and (4) \mbox{ReLIFT}~\citep{ma2025learning}: alternating between SFT and RL training across different batches, thereby dynamically switching the optimization objective.

\end{itemize}

\begin{table*}[t]
\centering
\caption{
Main experiment results on mathematical and general reasoning benchmarks based on \textbf{Qwen2.5-Math-7B}. 
\textbf{Bold} and \underline{underline} indicate the best and second-best results, respectively. $^*$~means the results are taken from the corresponding paper.}
\label{tab:main}
\setlength{\tabcolsep}{2.5pt}  
\renewcommand{\arraystretch}{1.3} 
\begin{center}    
\resizebox{\textwidth}{!}{%
\begin{tabular}{lccccc>{\columncolor{yellow!20}}c|cc>{\columncolor{cyan!20}}c}
\toprule
\multirow{2}{*}{\textbf{Model}} & \multicolumn{6}{c}{\textbf{Mathematical Reasoning}} & \multicolumn{3}{c}{\textbf{General Domain Reasoning}} \\
\cmidrule(lr){2-7} \cmidrule(lr){8-10}
 & \textbf{AIME2024} & \textbf{AMC} & \textbf{MATH-500} & \textbf{Minerva} & \textbf{Olympiad} & \textbf{Avg.} & \textbf{ARC-c} & \textbf{MMLU-Pro} & \textbf{Avg.} \\
\midrule
Qwen2.5-Math-7B & 11.9 & 33.7 & 47.0 & 12.5 & 21.9 & 25.4 & 34.5 & 23.3  & 28.9 \\
Qwen2.5-Math-7B-Instruct      
  & 11.3 & 48.2 & 82.6 & 36.8 & 39.7 & 43.7 & 72.4  & 38.3  &  55.4 \\
\midrule
\multicolumn{10}{c}{Pure RL without External Expert Guidance} \\
\midrule
GRPO & 24.0 & 59.0 & 84.0 & \underline{39.3} & 45.8 & 50.4 & \underline{80.5} &  47.2 & 63.9 \\
PRIME-Zero$^*$ & 17.0 & 54.0 & 81.4 & 39.0 & 40.3 & 46.3 & 73.3 &  32.7 & 53.0 \\
SimpleRL-Zero$^*$                
  & 27.0 & 54.9 & 76.0 & 25.0 & 34.7 & 43.5 & 30.2 &   34.5  &  32.4 \\
OpenReasoner-Zero$^*$  & 16.5 & 52.1 & 82.4 & 33.1 & 47.1 & 46.2 &  66.2 &  \textbf{58.7} & 62.5 \\
Oat-Zero$^*$                       
  & \underline{33.4} & 61.2 & 78.0 & 34.6 & 43.4 & 50.1 & 70.1 &   41.7 & 55.9 \\
\midrule
\multicolumn{10}{c}{RL Incorporating External Expert Guidance} \\
\midrule
SFT  & 27.7  & 56.0 & 84.8 & 38.2 & 44.7 & 50.3 &  51.5 &  33.1 &  42.3  \\
SFT-then-RL  & \textbf{33.5} & 62.3 & 86.6 & 41.2 & 47.6 & 54.3 & 52.9 & 36.1 & 44.5  \\
ReLIFT & 28.2 & 64.9 & 87.4 & 33.8 & 52.5 & 53.4 & 76.2 & 52.5 & 64.4 \\
LUFFY  & 29.4 & \underline{65.5} & \underline{88.4} & 38.2 & \underline{56.0} & \underline{55.5} & \underline{80.5} & \underline{53.0} & \underline{66.7} \\
\midrule
\multicolumn{10}{c}{\textbf{Our Method}} \\
\midrule
\mySolution  &  28.3 &  \textbf{66.2} & \textbf{89.2} & \textbf{41.5} & \textbf{57.6} &  \textbf{56.6} & \textbf{83.7} & 52.8 &  \textbf{68.3} \\
\bottomrule
\vspace{-20pt}
\end{tabular}%
}
\end{center}

\end{table*}

\subsection{Main Results}
\label{sec:main-results}
\paragraph{Mathematical reasoning performance}
As shown in Table~\ref{tab:main}, \mySolution\ achieves an average score of \textbf{56.6} across five mathematical reasoning benchmarks, outperforming all baselines. In particular, \mySolution\ yields improvements of \textbf{+6.3} and \textbf{+6.2} over \textit{SFT} and \textit{GRPO}, respectively, and a \textbf{+2.3} gain over the \textit{SFT-then-RL} baseline. These results validate our core hypothesis: \mySolution\ effectively enables the model to both internalize expert skills and leverage guidance for superior exploration, leading to a more robust acquisition of reasoning abilities.


\paragraph{General-domain reasoning performance}
On the two general reasoning benchmarks, \mySolution\ attains an average score of \textbf{68.3}, surpassing all baselines. By contrast, \textit{SFT} and \textit{SFT-then-RL} exhibit notably lower scores on these benchmarks, indicating that \mySolution, while leveraging external guidance, does not confine the model to rigid reasoning patterns; instead, it yields stronger generalization.

\begin{figure}[t!]
\begin{center}
\includegraphics[width=\linewidth]{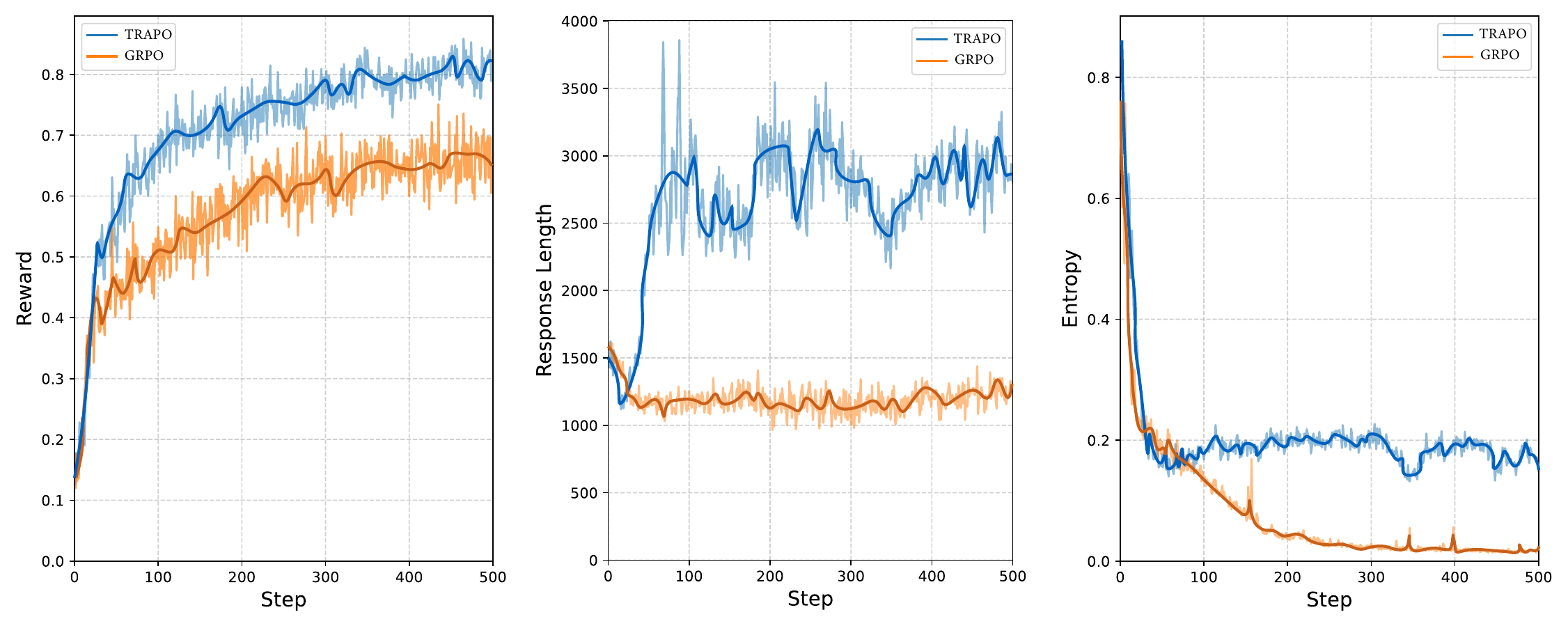}
\end{center}
\vspace{-2em}
\caption{Training dynamics of \mySolution\ compared with GRPO. From left to right: average reward, generation length, and output entropy during training. 
For fair comparison, both reward and generation length are computed by excluding trajectories guided by expert prefixes.
}
\label{fig:training-dynamic}
\end{figure}

\paragraph{Training dynamics}
Figure~\ref{fig:training-dynamic} presents a comparative analysis of the training dynamics between \mySolution\ and GRPO, revealing three key advantages of our approach: 
(1) \mySolution\ consistently achieves higher rewards throughout the training process and finally converges to a significantly higher final reward level.
(2) The generation length curve highlights that \mySolution\ rapidly increases its output length in the early stages, indicating a swift internalization of the expert's extended reasoning patterns. In contrast, GRPO struggles to produce longer solutions, consistently maintaining a short output length. (3) While both methods show an initial drop in policy entropy, their long-term behavior differs. \mySolution\ stabilizes at a relatively higher entropy level. We attribute this to its ability to maintain a dynamic balance: it simultaneously refines its own high-probability reasoning paths while remaining open to learning from externally provided, potentially low-probability expert guidance.

\begin{wrapfigure}{r}{0.35\textwidth}
    \centering
    \vspace{-2em}
    \includegraphics[width=0.35\textwidth]{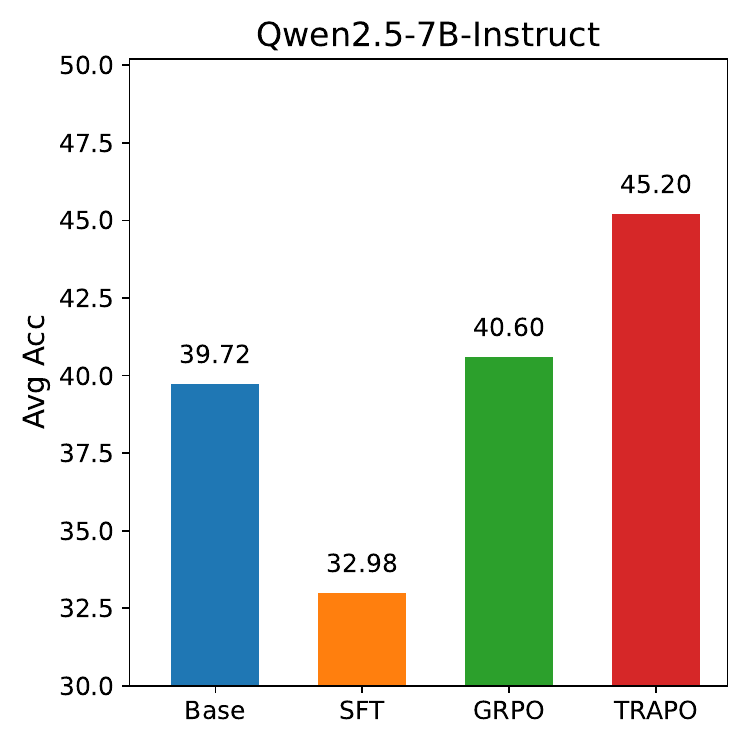}
    \vspace{-2em}
    \caption{Average accuracy across five mathematical benchmarks.}
    \label{fig:qwen2.5-7b-instruct}
    \vspace{-25pt}
\end{wrapfigure}
\paragraph{Extending to general-purpose LLMs.}
To further validate the generality of \mySolution, we conduct additional experiments on Qwen2.5-7B-Instruct, a general-purpose instruction-tuned model. As shown in Figure~\ref{fig:qwen2.5-7b-instruct}, \mySolution\ still achieves higher average accuracy than all other baselines across five mathematical reasoning benchmarks.

\subsection{Ablation Study}
\label{sec:ablation}
To assess the contributions of micro-group sampling and TrSFT, we conduct an ablation study on Qwen2.5-Math-7B (Table~\ref{tab:ablation}). Results show that micro-group sampling alone, even without explicit prefix learning, surpasses GRPO by adaptively determining prefix length to enhance reasoning and reward density. Adding TrSFT further improves performance: unlike standard SFT loss (which degrades) or the offline RL loss from LUFFY (which brings limited gains), TrSFT effectively internalizes expert prefixes. We also analyze the effect of the trust-region parameter $\alpha$ in Appendix~\ref{appendix:hyperparameter}.
\begin{table}[t]
\centering
\caption{Ablation study on \mySolution\ components.}
\label{tab:ablation}
\vspace{-10pt}
\begin{center}    
\resizebox{\textwidth}{!}{
\begin{tabular}{lcccccc}
\toprule
\textbf{Model} & \textbf{AIME2024}& \textbf{AMC} & \textbf{MATH-500} & \textbf{Minerva} & \textbf{Olympiad} & \textbf{Avg.}  \\
\midrule
GRPO (Qwen2.5-Math-7B) & 24.0 & 59.0 & 84.0 & 39.3 & 45.8 & 50.4  \\
~~~ + Micro-group sampling & 26.0 & 63.7 & 84.2 & 39.3 & 50.1 & 52.7  \\
~~~ + Micro-group sampling + \textit{SFT Loss} & 14.6 & 32.8 & 60.2 & 25.4 & 28.5 & 32.3  \\
~~~ + Micro-group sampling + \textit{LUFFY Loss} & 26.2 & 65.1 & 87.2 & 36.8 & 52.5 & 53.6  \\
~~~ + Micro-group sampling + \textit{\mySFT\ Loss} & 28.3 & 66.2 & 89.2 & 41.5 & 57.6 & 56.6
\\\bottomrule
\end{tabular}}
\end{center}
\end{table}

\subsection{Test-Time Scaling}

We evaluate pass@k, the success rate over k independent rollouts, to better estimate the upper bound of model capability~\citep{snell2024scaling}, as recent studies show that multiple generation attempts reveal reasoning potential more accurately than few rollouts~\citep{yue2025does, wang2022self}.


Figure~\ref{fig:test-time-scale} illustrates the pass@k performance on the AIME2024 benchmark, from which we derive two key insights: (1) We observe that the base model (Qwen2.5-Math-7B) surpasses the GRPO-trained model when evaluated with a sufficiently large $k$. This aligns with prior findings~\citep{yue2025does} and suggests that standard RL primarily stimulates the model to select better solutions from its existing knowledge space, but does not fundamentally expand that space with new problem-solving skills. (2) Both \mySolution\ and the SFT-based methods demonstrate strong performance scaling with larger $k$, indicating they possess a richer underlying solution space. The superior performance of \mySolution\ highlights its success in effectively internalizing the external knowledge from expert trajectories, thereby expanding the model's intrinsic capabilities.

\begin{wrapfigure}{r}{0.4\textwidth}
    \centering
    \vspace{-6.0em}
    \includegraphics[width=0.4\textwidth]{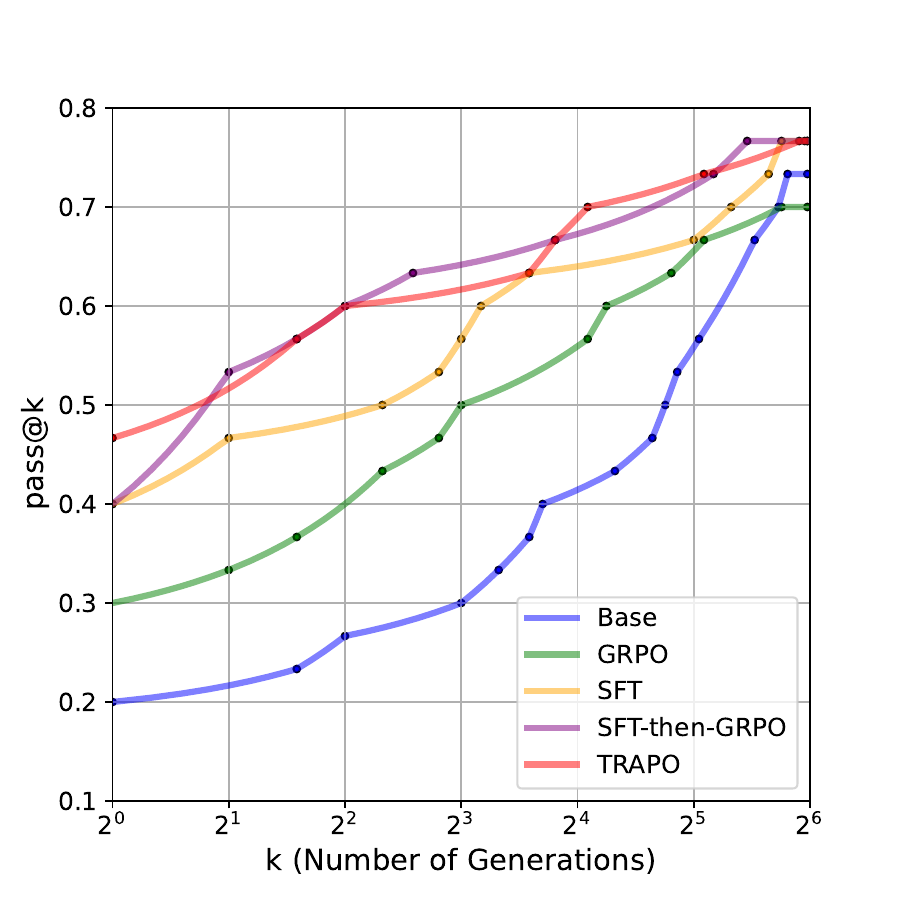}
    \vspace{-1.5em}
    \caption{Comparison of pass@k results on the AIME2024 dataset
    }
    \label{fig:test-time-scale}
\end{wrapfigure}

%% file: chapters/related_works.tex
\section{Related Works}

\paragraph{RL for LLM reasoning}

RL has recently driven major advances in LLM reasoning, as demonstrated by milestone systems such as OpenAI-o1~\citep{jaech2024openai}, DeepSeek-R1~\citep{guo2025deepseek}, and Kimi-1.5~\citep{team2025kimi}. Follow-up studies have examined RL from three complementary perspectives. (1) Empirical analyses investigate its effect on reasoning, e.g., \citet{yue2025does} shows that RL refines existing solution trajectories rather than expanding reasoning capabilities. (2) Data-centric approaches reshape training signals, such as R3~\citep{xi2024training} with reverse curriculum learning, and ADARFT~\citep{shi2025efficient} or Logic-R1~\citep{xie2025logic} with dynamic difficulty scheduling to improve sample efficiency. (3) Optimization methods refine RL objectives, from PPO’s clipped trust region and GRPO’s group-based advantage estimation to recent variants like DAPO~\citep{yu2025dapo}, Dr.GRPO~\citep{liu2025understanding}, and VAPO~\citep{yue2025vapo}. Distinct from these directions, \mySolution\ enhances RL efficiency by adaptively incorporating expert demonstrations, which both guide exploration and act as direct supervision, simultaneously improving reward density and reasoning capability.



\paragraph{Combining SFT with RL}
While the traditional training paradigm primarily followed a serial approach of first applying SFT and then RL, recent works have shifted focus towards concurrently performing both SFT and RL. A straightforward approach is to directly combine the SFT and RL losses, adjusting their respective weights. For instance, SRFT~\citep{fu2025srft} proposes a single-stage training method that dynamically adjusts the weights of SFT and RL losses based on token entropy. AMFT~\citep{he2025amft} views the balance between SFT and RL as a learnable parameter and introduces a meta-gradient adaptive weight controller. HPT~\citep{lv2025towards}, on the other hand, dynamically decides whether to apply SFT or RL based on the model's rollout performance, with binary weights of 0 or 1. Another approach involves interleaving SFT within the RL pipeline, as seen in ReLIFT~\citep{ma2025learning}, which collects poorly performing samples during RL training and stores them in a buffer to later apply SFT. Additionally, LUFFY~\citep{yan2025learning}, inspired by offline RL, treats one expert trajectory as offline data and mixes it with the remaining seven online trajectories in a group, using an importance ratio to calibrate the distribution shift, achieving state-of-the-art performance. Similar to our work, Prefix-RFT~\citep{huang2025blending} samples prefixes as guidance and uses entropy to select expert tokens for SFT. These works collectively emphasize the significant benefits of combining these two post-training paradigms. We present the first study to theoretically investigate the challenge of combining SFT and RL objectives, and propose a simple but effective solution to resolve the distribution-blending effect and foster synergy with the RL objective.

%% file: chapters/appendix.tex
\appendix
\section{Derivations and Proofs}
\subsection{Derivation of the Equivalence between SFT and KL}
\label{appendix:sft-kl-eq}
Following~\citet{agarwal2024policy}, we remove the length normalization factor and define the discrepancy between token-level distribution of $\tea$ and $\stu$ over a given sequence $\boldsymbol{y}$:
\begin{equation}
  \D\big(\tea \| \stu^\theta\big)(\boldsymbol{y}|\boldsymbol{x}) := \sum_{n=1}^{|\boldsymbol{y}|} \D\big(\tea(\cdot|y_{<n},\boldsymbol{x}) \| \stu^\theta(\cdot|y_{<n}, \boldsymbol{x})\big).\label{eq:div}
\end{equation}
The problem then becomes proving the following equivalence:
\begin{equation}
    \min_{\theta} L_{\text{SFT}}(\theta) \Leftrightarrow \min_{\theta} \mathbb{E}_{x\sim X}\Big[\mathbb{E}_{\boldsymbol{y}\sim \tea(\cdot|\boldsymbol{x})}\big[ \D_\text{KL}\big(\tea \| \stu^{\theta}\big)(\boldsymbol{y}|\boldsymbol{x})\big]\Big]. \label{eq:sft-kl-eq}
\end{equation}
We start by substituting Eq.~\eqref{eq:div} into the objective expression on the right-hand side of Eq.~\eqref{eq:sft-kl-eq}:
\begin{align}
    \mathbb{E}_{\boldsymbol{y}\sim \tea(\cdot|\boldsymbol{x})}\big[ \D_\text{KL}\big(\tea \| \stu^{\theta}\big)(\boldsymbol{y}|\boldsymbol{x})\big] &=  \mathbb{E}_{\boldsymbol{y}\sim \tea(\cdot|\boldsymbol{x})}\big[\sum_{n=1}^{|\boldsymbol{y}|}\D_\text{KL}\big(\tea(\cdot|y_{<n},\boldsymbol{x})\|\stu^\theta(\cdot|y_{<n}, \boldsymbol{x})\big)\big] \\&=\mathbb{E}_{\boldsymbol{y}\sim \tea(\cdot|\boldsymbol{x})}\big[\sum_{n=1}^{|\boldsymbol{y}|}\sum_{y_n\in V}\tea(y_n|y_{<n},\boldsymbol{x})\log\frac{\tea(y_n|y_{<n},\boldsymbol{x})}{\stu^{\theta}(y_n|y_{<n},\boldsymbol{x})}\big].
\end{align}
We then remove the term that are independent of $\theta$:
\begin{align}
&\min_\theta \mathbb{E}_{\boldsymbol{x}\sim X}\Big[\mathbb{E}_{\boldsymbol{y}\sim \tea(\cdot|\boldsymbol{x})}\big[\sum_{n=1}^{|\boldsymbol{y}|}\sum_{y_n\in V}\tea(y_n|y_{<n},\boldsymbol{x})\log\frac{\tea(y_n|y_{<n},\boldsymbol{x})}{\stu^{\theta}(y_n|y_{<n},\boldsymbol{x})}\big]\Big] \\ \Leftrightarrow &\min_\theta \mathbb{E}_{\boldsymbol{x}\sim X}\Big[\mathbb{E}_{\boldsymbol{y}\sim \tea(\cdot|\boldsymbol{x})}\big[\sum_{n=1}^{|\boldsymbol{y}|}\sum_{y_n\in V}-\tea(y_n|y_{<n},\boldsymbol{x})\log\stu^{\theta}(y_n|y_{<n},\boldsymbol{x})\big]\Big] \\ \Leftrightarrow
&\min_\theta \mathbb{E}_{\boldsymbol{x}\sim X}\Big[\mathbb{E}_{\boldsymbol{y}\sim \tea(\cdot|\boldsymbol{x})}\big[\sum_{n=1}^{|\boldsymbol{y}|}\mathbb{E}_{y_n\sim\tea(\cdot|y_{<n},\boldsymbol{x})}[-\log\stu^{\theta}(y_n|y_{<n},\boldsymbol{x})]\big]\Big] \\ \Leftrightarrow &\min_\theta\mathbb{E}_{\boldsymbol{x}\sim X}\Big[\mathbb{E}_{\boldsymbol{y}\sim\tea(\cdot|\boldsymbol{x})}\big[\sum_{n=1}^{|\boldsymbol{y}|}[-\log\stu^\theta(y_n|y_{<n},\boldsymbol{x})]\big]\Big] \\
\Leftrightarrow &\min_\theta\mathbb{E}_{\boldsymbol{x}\sim X}\Big[\mathbb{E}_{\boldsymbol{y}\sim\tea(\cdot|\boldsymbol{x})}\big[-\log \stu^\theta(\boldsymbol{y}|\boldsymbol{x})\big]\Big]
\end{align}
The last line follows from the equivalence between the logarithm of a product and the sum of logarithms.

\subsection{Proof of Proposition \ref{prop:solution}}
\label{appendix:solution}
\Firstprop*
\begin{proof}
First, we present the equivalent form of this optimization problem, with the derivation similar to that in Eq.~\eqref{eq:sft-kl-eq}:
\begin{equation} \label{eq:TrSFT-composite}
\begin{aligned}
    \min_{\theta} L_{\text{\mySFT}}^\alpha(\theta) &\Leftrightarrow \min_{\theta} \mathbb{E}_{x\sim X}\Big[\mathbb{E}_{\boldsymbol{y}\sim \tea(\cdot|\boldsymbol{x})}\big[ \D_\text{\mySFT}^\alpha\big(\tea \| \stu^{\theta}\big)(\boldsymbol{y}|\boldsymbol{x})\big]\Big],\\ 
    \text{where} \quad \D^\alpha_{\text{\mySFT}}(\tea\|\stu^\theta) &= \sum_{c \in \mathcal{C}} \tea(c) \ell^\alpha\big(\stu^\theta(c)\big) \quad \text{and} \quad \ell^\alpha(p) = 
    \begin{cases}
        - \dfrac{p}{\alpha}+1-\log\alpha, & p < \alpha, \\[8pt]
        - \log p, & p \geq \alpha,
    \end{cases}
    \nonumber
\end{aligned}
\end{equation}

Then we apply the Karush–Kuhn–Tucker (KKT)~\citep{ghojogh2021kkt} conditions.

We introduce multiplier $\lambda \in R$ for the equality constraint and multipliers $\big\{\mu_c|c \in \mathcal{C}\big\}$ for nonnegativity constraints. The Lagrangian is
\begin{equation}
    \mathcal{L}\big(\stu(c_1),\cdots,\stu(c_N),\lambda,\mu_{c_1},\cdots,\mu_{c_N}\big) = \sum_{c}\tea(c)\ell^\alpha\big(\stu(c)\big) - \sum_{c} \mu_c \stu(c) + \lambda \big(\sum_c \stu(c) - 1\big).
\end{equation}
For each $c$, stationarity gives $\tea(c)\nabla_{\stu(c)}\ell^\alpha\big(\stu(c)\big) - \mu_c + \lambda = 0$, together with $\mu_c \ge 0$ and $\mu_c \stu^c = 0$.
The derivative of $\ell^\alpha$ is
\[
\nabla_{\stu}\ell^\alpha(\stu) =
\begin{cases}
-\dfrac{1}{\alpha}, & 0 \le \stu \le \alpha, \\[6pt]
-\dfrac{1}{\stu}, & \stu > \alpha .
\end{cases}
\]
\textbf{Case A \big($\stu(c) > \alpha$\big).}

In this case, $\nabla\ell^\alpha\big(\stu(c)\big) = -1/\stu(c)$. The stationarity condition with $\mu_c=0$ gives
\[
-\frac{\tea(c)}{\stu(c)} + \lambda = 0 
\quad\Rightarrow\quad 
\stu(c) = \frac{\tea(c)}{\lambda}.
\]
Since we assumed $\stu(c) > \alpha$, this requires $\tea(c) > \alpha\lambda$.

\medskip

\textbf{Case B \big($0 < \stu(c) \le \alpha$\big).} 

Here $\nabla\ell^\alpha(\stu(c)) = -1/\alpha$. With $\mu_c=0$, the condition becomes
\[
-\frac{\tea(c)}{\alpha} + \lambda = 0 
\quad\Rightarrow\quad 
\tea(c) = \alpha\lambda.
\]
Thus interior solutions in $(0,\alpha]$ are only possible if $\tea(c) = \alpha\lambda$.

\medskip
\noindent\textbf{Case C \big($\stu(c) = 0$\big).}

At the boundary $\stu(c)=0$, we use $\nabla\ell^\alpha(0^+)=-1/\alpha$. 
Stationarity then gives
\[
-\frac{\tea(c)}{\alpha} + \lambda - \mu_c = 0
\quad\Rightarrow\quad 
\mu_c=\lambda - \frac{\tea(c)}{\alpha}.
\]
For feasibility we require $\mu_c\ge 0$, hence this case occurs only when 
$\tea(c) \le \alpha\lambda$.

Collecting the above conditions, the optimal solution must satisfy:
\[
\stu^{\star}(c) =
\begin{cases}
\dfrac{\tea(c)}{\lambda}, & \text{if } \tea(c) > \alpha\lambda, \\[10pt]
0, & \text{if } \tea(c) < \alpha\lambda, \\[10pt]
\text{any value in } [0,a], & \text{if } \tea(c) = \alpha\lambda.
\end{cases}
\]

In realistic token distributions the set of tokens satisfying $\tea(c)=\alpha\lambda$ has measure zero and can be ignored. Equivalently, one may simply enforce $\stu(c)=0$ for all such tokens. The optimality conditions derived above remain valid. Then from $\sum_c \stu(c)=1$, we can derive $\lambda =\sum_{c \in S(\lambda)}\tea(c)$.

Observe that the normalization condition requires
\[
\sum_{c \in S(\lambda)} \frac{\tea(c)}{\lambda} = 1.
\]
As $\lambda$ increases continuously from $0$ to $1$, 
the left-hand side decreases monotonically from $+\infty$ to the sum of probablity less than $1$, 
while the right-hand side remains constant at $1$. 
Therefore, by the intermediate value theorem, such a value of $\lambda$ must exist.

\end{proof}

\section{Pseudocode for \mySolution\ }
\label{appendix:pseudocode}
Algorithm~\ref{alg:tag} summarizes the complete training procedure of {\mySolution}.

\begin{algorithm}[t]
\caption{\mySolution\ : Micro-group Sampling with Adaptive Guidance}
\label{alg:tag}
\begin{algorithmic}[1]
\Require Dataset $D_\text{data}=\{(\boldsymbol{x}, \boldsymbol{y^\star})\}$; expert policy $\tea$; target policy $\stu^\theta$; trust-region parameter $\alpha$; number of micro-groups $N$; metadata $\{(n_i, L_i, t_i)\}_{i=1}^N$ with $0=L_1 < L_2 < \cdots < L_N = 1$
\Ensure Updated student parameters $\theta$

\Function{PassRate}{$\mathcal{S}$}
    \State \textbf{return} $\frac{1}{\max(1,|\mathcal{S}|)}\sum_{(\boldsymbol{x},\boldsymbol{y})\in\mathcal{S}} \mathds{1}[\text{Correct}(\boldsymbol{y}; \boldsymbol{y^\star})]$ \Comment{compare to ground truth $\boldsymbol{y^\star}$}
\EndFunction

\For{each mini-batch $\mathcal{B}\subset D_\text{data}$}
    \For{each prompt $\boldsymbol{x} \in \mathcal{B}$}
        \State $\mathcal{S} \gets \emptyset$ \Comment{collected samples for this $x$ so far}
        \For{$i = 1$ to $N$} \Comment{serial micro-groups}
            \State $pr_i \gets \Call{PassRate}{\mathcal{S}}$ \Comment{use samples from previous $1,\ldots ,i\!-\!1$ micro-groups}
            \If{$pr_i \le t_i$} \Comment{inject expert guidance}
                \State $\boldsymbol{y} \gets \tea(\cdot \mid \boldsymbol{x})$
                \State $\boldsymbol{y^{<L_i}} \gets \text{first } \lfloor L_i \cdot |\boldsymbol{y}| \rfloor \text{ tokens of } \boldsymbol{y}$
                \For{$k=1$ to $n_i$}
                    \State $\boldsymbol{\hat{y}} \sim \stu^\theta(\cdot \mid \boldsymbol{x} , \boldsymbol{y^{<L_i}})$
                    \State $\mathcal{S} \gets \mathcal{S} \cup \{(\boldsymbol{x},\boldsymbol{y^{<L_i}} \oplus \boldsymbol{\hat{y}})\}$
                \EndFor
            \Else \Comment{no guidance}
                \For{$k=1$ to $n_i$}
                    \State $\boldsymbol{\hat{y}} \sim \stu^\theta(\cdot \mid \boldsymbol{x})$
                    \State $\mathcal{S} \gets \mathcal{S} \cup \{(\boldsymbol{x},\boldsymbol{\hat{y}})\}$
                \EndFor
            \EndIf
        \EndFor
    \EndFor
    \State Compute $\nabla_\theta L_{\text{TrSFT}}^\alpha$ on \emph{guided tokens only} (those fed to the target) via Eq.~\eqref{eq:TrSFT-gradient}
    \State Compute $\nabla_\theta L_{\text{GRPO}}$ on \emph{self-generated trajectories}
    \State $\nabla_\theta L \gets \nabla_\theta L_{\text{TrSFT}}^\alpha + \nabla_\theta L_{\text{GRPO}}$ 
    \State $\theta \gets \text{OptimizerStep}(\theta, \nabla_\theta L)$
\EndFor
\end{algorithmic}
\end{algorithm}

\section{Experimental Datails}
\subsection{Preliminary Experiment}
\label{appendix:pre-exp}
\paragraph{Data Collection}
We use DeepSeek-R1 to generate eight solutions for each problem in the MATH-500 benchmark. For each solution, we extract prefixes with varying token proportions as guidance, which are then concatenated to the original prompt. At each token proportion level, we obtain 4,000 prompts, which are subsequently fed into Qwen2.5-3B-Instruct for continuation.

\paragraph{Evaluation}
\begin{itemize}
    \item Accuracy is computed by comparing the generated answers with the ground truth using Math-Verify .
    \item We further input the generated continuations into GPT-4o-mini to count the occurrences of two distinct reasoning behaviors. The specific prompt templates used are as follows:
\end{itemize}

\begin{tcolorbox}[
    center,
    arc=0mm,
    boxrule=1pt,
    colback=blue!6!white,
    colframe=black,
    colbacktitle=black,
    title=Backtracking,
    boxed title style={boxrule=0pt,colframe=white}
]
You are given a math reasoning problem and the latter half of a reasoning trace generated by Qwen2.5-3B-Instruct.

\textbf{Problem:}  
\{problem\_text\}

\textbf{Model output:}  
\{output\_text\}

Detect BACKTRACKING: moments where the model abandons a line of attack and starts a new, distinct approach (not merely the next algebraic step of the same plan).  
Count an instance when the text signals a restart/switch such as:
\begin{itemize}
    \item "That doesn't work / leads nowhere / contradiction, so try...", "Instead, I'll...", "Another approach:", "Restart with...", "Consider a different method/case".
    \item Dropping the current construction and trying a fresh one (new substitution, inequality, identity, factorization, case split that resets the plan).
\end{itemize}

Do NOT count routine sequential steps in one coherent plan as backtracking.  
Return only the number between \texttt{<count>} and \texttt{</count>}. If none, return \texttt{<count>0</count>}.
\end{tcolorbox}

\begin{tcolorbox}[
    center,
    arc=0mm,
    boxrule=1pt,
    colback=blue!6!white,
    colframe=black,
    colbacktitle=black,
    title=Backward Chaining,
    boxed title style={boxrule=0pt,colframe=white}
]
You are given a math reasoning problem and the latter half of a reasoning trace generated by Qwen2.5-3B-Instruct.

\textbf{Problem:}  
\{problem\_text\}

\textbf{Model output:}  
\{output\_text\}

Detect BACKWARD-CHAINING: reasoning that starts from (or is framed by) the desired conclusion/target and derives necessary conditions that would make it true.  
Count an instance when the text says things like:
\begin{itemize}
    \item "We want to show P, so it suffices to show Q", "Work backwards from the target...", "If the result were true, then ... must hold", "To get X, we need Y", rearranging *from the goal form to requirements*.
    \item Proof-by-contradiction setup qualifies if it explicitly frames the goal via assuming its negation to derive impossibility.
\end{itemize}

Do NOT count ordinary forward algebra unless it is clearly framed as working from the goal backward.  
Return only the number between \texttt{<count>} and \texttt{</count>}. If none, return \texttt{<count>0</count>}.
\end{tcolorbox}

\subsection{Main Experiment}
\label{appendix:main_exp}
In this section, we provide training details not mentioned in the main text.
\paragraph{Base Model}
Our main experiments are conducted using Qwen2.5-Math-7B, a model designed for mathematical reasoning tasks. Due to the limitation of 8k tokens in the DeepSeek-R1 expert trajectories and the 4096 context length of Qwen2.5-Math-7B, we follow LUFFY's approach and increase the model's rope theta from 10000 to 40000.

\paragraph{Dataset}
Our data source is OpenR1-Math-46k-8192~\citep{yan2025learning}, which filters out incorrect trajectories and those with token lengths greater than 8192 from OpenR1-Math-200k~\citep{openr1}. Since OpenR1-Math-200k contains at least two trajectories generated by Deepseek-R1 for each problem, we also collect the corresponding second trajectory for each entry in OpenR1-Math-46k-8192, resulting in approximately 46k prompts and 92k trajectories. Furthermore, since we did not require the model to split the output into think and answer parts in the system prompt, we retain only the tokens between the \texttt{<think>} and \texttt{</think>} tags from the original Deepseek-R1 output as candidate trajectories. During micro-group sampling, we first calculate the corresponding token position based on the desired ratio, then locate the first reasoning step delimiter (typically a double newline in the Deepseek-R1 style) after this position to ensure the atomicity and completeness of the reasoning steps.

\paragraph{SFT}
For all SFT methods, we train on the 92k \texttt{<}prompt, trajectory\texttt{>} pairs for 2 full epochs as described above. The hyperparameters are set as follows: sequence length of 16,384, learning rate of 5e-5 with a warmup ratio of 0.1, and a batch size of 32. The remaining experimental settings are consistent with those used in OpenR1-Qwen-7B~\citep{openr1}.
All SFT training is conducted using the OpenRLHF~\citep{hu2024openrlhf} framework.

\paragraph{RL}
All RL methods use a variant of GRPO, Dr.GRPO~\citep{liu2025understanding}, which removes length normalization and the standard deviation regularization in advantage calculation. The hyperparameters are as follows: batch size of 128, 8 rollouts per group, max prompt length of 1024, max response length of 8192, and a constant learning rate of 1e-6. We use vLLM~\citep{kwon2023efficient} for rollouts during training with a sampling temperature set to 1.0. RL training is conducted using the verl~\citep{sheng2025hybridflow} framework.

\paragraph{\mySolution\ }
Since the dataset contains only two expert trajectories per prompt, and we provide expert guidance for up to 4 out of 8 rollouts in each group, we randomly select one expert trajectory as a potential guidance source for each micro-group to ensure diversity in the expert data. Except for setting the trust region parameter \(\alpha\) to \rev{0.1}, all other training parameters are consistent with those used in RL and are implemented within the verl framework.

\paragraph{Evaluation} For evaluation, we use
a lower temperature of 0.6 and a maximum generation length of 8,192 tokens. We employ Math-verify and OAT-Grader~\citep{liu2025oat} frameworks for evaluation.

\paragraph{Prompt Template}

Following SimpleRL~\citep{zeng2025simplerl}, we use the simplest prompt template across all methods and models to ensure that the improvement in reasoning capability comes solely from the training. The template is as follows:

\begin{tcolorbox}[
    center,
    arc=0mm,
    boxrule=1pt,
    colback=blue!6!white,
    colframe=black,
    colbacktitle=black,
    title=Prompt Template,
    boxed title style={boxrule=0pt,colframe=white}
]
\texttt{<|}im\_start\texttt{|>}system \\
You are a helpful assistant. \\
\texttt{<|}im\_end\texttt{|>} \\ \\
\texttt{<|}im\_start\texttt{|>}user \\
\{problem\}

Please reason step by step, and put your final answer within \textbackslash boxed\{\}. \\
\texttt{<|}im\_end\texttt{|>}
\end{tcolorbox}

\section{Hyperparameter Study}
\label{appendix:hyperparameter}

\subsection{Ablation of Trust-Region Parameter}
\begin{figure}[t!]
\begin{center}
\includegraphics[width=\linewidth]{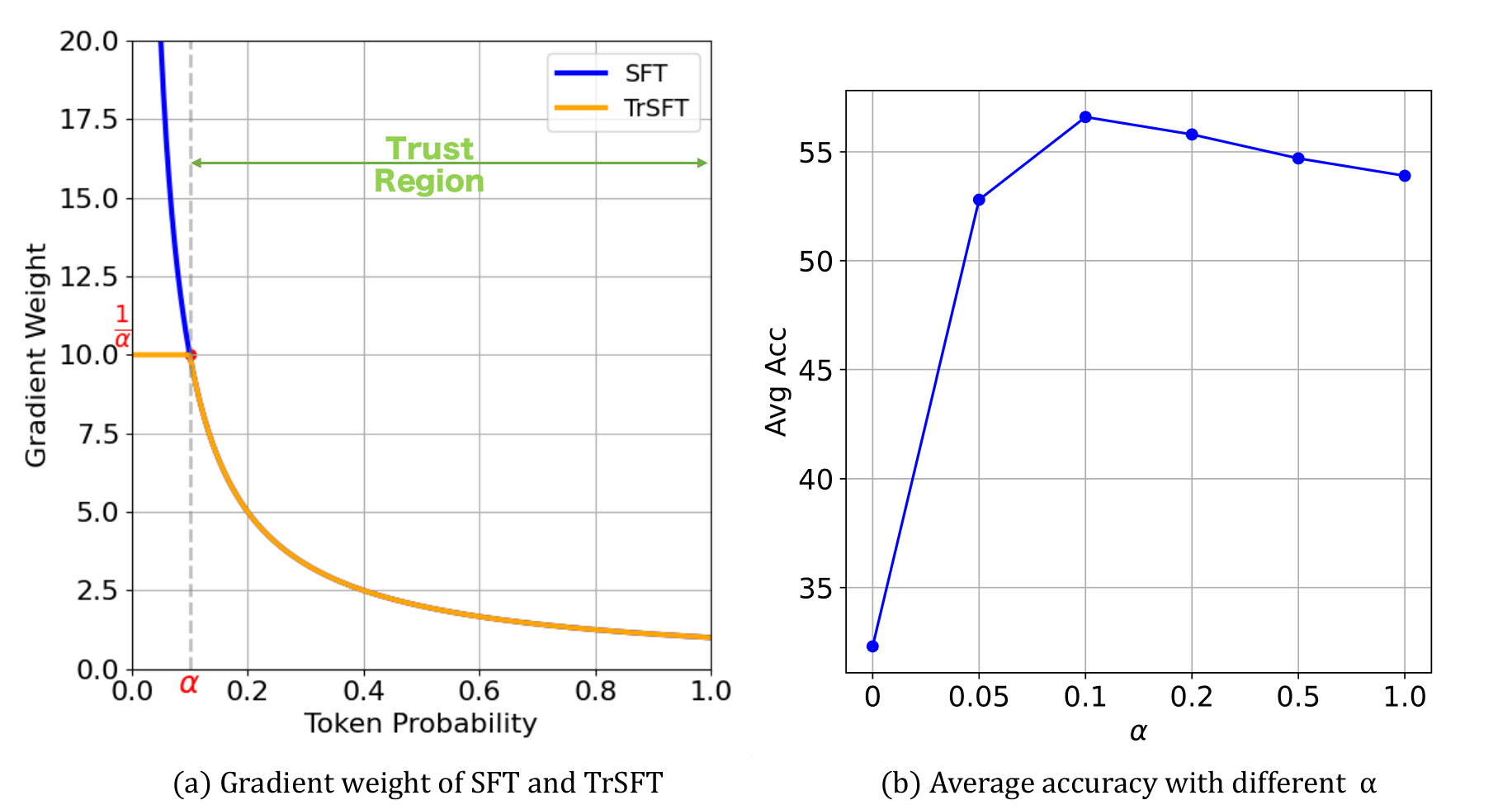}
\end{center}
\vspace{-1em}
\caption{(a) The plot showing the change in gradient weight with respect to token probability in SFT and TrSFT. (b) The average accuracy on five mathematical reasoning benchmarks for Qwen2.5-Math-7B trained with \mySolution\ for different values of \(\alpha\).}
\label{fig:hyperparameter}
\end{figure}

We apply equations \ref{eq:SFT-gradient} and \ref{eq:TrSFT-gradient} to a single token, which leads to the following form:
\begin{equation}
    \nabla_{\theta} L = \frac{1}{\stu^{\theta}(y_n|\boldsymbol{x},\boldsymbol{y_{<n}}) \texttt{ or } \max\big(\stu^{\theta}(y_n|\boldsymbol{x},\boldsymbol{y_{<n}}), \alpha\big)}\nabla_{\theta}\stu^{\theta}(y_n|\boldsymbol{x},\boldsymbol{y_{<n}}). \label{eq:compare-gradient}
\end{equation}
The weight (the fractional factor) in front of the gradient term varies with \(\stu^{\theta}(y_n|\boldsymbol{x}, \boldsymbol{y_{<n}})\) as shown in Figure~\ref{fig:hyperparameter}(a). When \(\alpha = 0\), TrSFT is equivalent to SFT, meaning that all tokens lie within the trust region. As \(\alpha\) increases, the gradient update weight of tokens outside the trust region is reduced to a fixed value of \(1/\alpha\), significantly mitigating the disruptive impact of large gradient updates from these tokens on the model parameters. However, when \(\alpha\) becomes too large, all tokens are assigned the same gradient weight, completely losing the characteristic of SFT where tokens with lower probability require more learning. Therefore, choosing an appropriate value of \(\alpha\) is crucial to balance the resistance against low-probability tokens and the need to effectively learn from them. Figure~\ref{fig:hyperparameter}(b) shows the training performance of \mySolution\  under different values of \(\alpha\) (0, 0.05, 0.1, 0.2, 0.5, 1). We observe that the best performance is achieved when \(\alpha = 0.1\), with both higher and lower values of \(\alpha\) failing to reach the optimal performance.

\subsection{Ablation of Micro-Group Hyperparameters}

\rev{In micro-group sampling, we adopt an adaptive strategy for determining expert prefix length. Unlike fixed-prefix approaches, this dynamic mechanism selects an appropriate amount of expert guidance per prompt, providing stronger assistance when unguided rollouts perform poorly.}

\rev{Each micro-group \(g_i\) contains several tunable hyperparameters:  
(1) the prefix-length proportion \(L_i\),  
(2) the reward threshold \(t_i\), and  
(3) the number and sizes of micro-groups.  
All variants were trained on Qwen2.5-Math-7B for 100 steps, with average accuracy reported across six mathematical reasoning benchmarks.}

\vspace{0.5em}
\noindent\textbf{\rev{1. Prefix-length proportion.}}

\begin{table}[h]
\centering
\caption{Ablation of the prefix-length proportion.}
\label{tab:ablation-prefix-length}
\vspace{6pt}
\begin{tabular}{c|cccc|c}
\toprule
Variant & $L_1$ & $L_2$ & $L_3$ & $L_4$ & Avg. Perf. \\
\midrule
A & 0 & 0.2 & 0.5 & 1.0 & 39.4 \\
B & 0 & 0.2 & 0.5 & 0.7 & 37.6 \\
C & 0.1 & 0.2 & 0.5 & 1.0 & 34.1 \\
D & 0 & 0.1 & 0.8 & 1.0 & 39.2 \\
E & 0 & 0.3 & 0.4 & 1.0 & 39.5 \\
\bottomrule
\end{tabular}
\end{table}

\rev{As shown in Table~\ref{tab:ablation-prefix-length}, compared with Variant A, Variants B and C confirm the importance of our two design principles:  
(1) the first micro-group must receive no expert prefix to preserve unguided exploration, and  
(2) the final micro-group should expose the full expert trajectory.  
When modifying only the intermediate prefix proportions (Variants D and E), performance remains largely unchanged, suggesting that TRAPO is robust to these internal ratios.}

\vspace{0.5em}
\noindent\textbf{\rev{2. Reward threshold.}}

\begin{table}[h]
\centering
\caption{Ablation of the reward threshold.}
\label{tab:ablation-reward-threshold}
\vspace{6pt}
\begin{tabular}{c|cccc|c}
\toprule
Variant & $t_1$ & $t_2$ & $t_3$ & $t_4$ & Avg. Perf. \\
\midrule
A & -1 & 0.5 & 0.7 & 0.9 & 39.4 \\
B & -1 & 0.5 & 0.5 & 0.5 & 38.1 \\
C & -1 & 0.9 & 0.7 & 0.5 & 37.4 \\
D & -1 & 0.55 & 0.75 & 0.95 & 38.9 \\
E & -1 & 0.45 & 0.65 & 0.85 & 39.3 \\
\bottomrule
\end{tabular}
\end{table}

\rev{We maintain a monotonically increasing threshold schedule to match the naturally increasing expected reward of longer prefixes.  
As shown in Table~\ref{tab:ablation-reward-threshold}, constant thresholds (B) or decreasing schedules (C) reduce performance.  
As long as thresholds increase in a reasonable range (D and E), their exact values have minimal effect.}

\vspace{0.5em}
\noindent\textbf{\rev{3. Number and sizes of micro-groups.}}
\vspace{-10pt}
\begin{table}[h]
\centering
\caption{Ablation of the number and sizes of micro-groups.}
\vspace{6pt}
\label{tab:ablation-micro-groups}
\begin{tabular}{c|c|c}
\toprule
Variant & $(n_1, n_2, \dots)$ & Avg. Perf. \\
\midrule
A & (4, 2, 1, 1) & 39.4 \\
B & (2, 4, 1, 1) & 37.2 \\
C & (4, 4) & 38.0 \\
D & (4, 3, 1) & 39.1 \\
\bottomrule
\end{tabular}
\end{table}

\rev{For the organization of micro-groups, two structural principles are crucial:  
(1) the first unguided micro-group must be sufficiently large to support exploration, and  
(2) at least one intermediate group should provide partial guidance to bridge unguided and fully guided rollouts.  
As shown in Table~\ref{tab:ablation-micro-groups}, shrinking the unguided group (B) or removing intermediate groups (C) harms performance.  
Configurations satisfying both principles (A and D) achieve similar accuracy.}

\section{Additional Cost and Efficiency Analysis}
\label{appendix:cost}

\rev{
This section provides further details on the computational efficiency of TRAPO, supplementing the discussion in the main paper. Due to the longer generation lengths and the sequential execution required by micro-group sampling, TRAPO incurs a higher per-step computational cost compared to standard GRPO. However, under matched wall-clock training-time budgets, TRAPO consistently achieves stronger reasoning performance, even though it executes slightly fewer training steps than the baselines. This suggests that TRAPO benefits from more effective training-time scaling.
}

\paragraph{\rev{Overall GPU-hour usage.}}
\rev{
Table~\ref{tab:gpu-hours} reports the total GPU hours consumed by each method. For the SFT-then-RL pipeline, this measurement includes both the SFT stage and the subsequent RL stage. We also report the average number of expert-prefix tokens (\emph{Expert Tokens}) and target-model-generated tokens (\emph{Target Tokens}) per RL trajectory.
}

\begin{table}[t]
\centering
\caption{GPU hours and token statistics across training methods.}
\vspace{6pt}
\label{tab:gpu-hours}
\begin{tabular}{lccc}
\toprule
Model & GPU hours & Expert Tokens & Target Tokens \\
\midrule
GRPO          & $68 \times 8$        & 0   & 1189 \\
SFT-then-RL   & $(42 + 105) \times 8$ & 0   & 4760 \\
LUFFY         & $81 \times 8$        & 519 & 2439 \\
TRAPO         & $89 \times 8$        & 501 & 2545 \\
\bottomrule
\end{tabular}
\end{table}

\paragraph{\rev{Training dynamics under matched wall-clock budgets.}}
\rev{
To assess training-time efficiency, we measure the number of steps completed by each method under equal GPU-hour budgets, together with the mean reward of the collected rollout batches. As reward is strongly correlated with downstream reasoning performance, it serves as a useful indicator of learning progress. For SFT-then-RL, the GPU-hour measurement below begins at the start of the RL stage.
}

\begin{table}[t]
\centering
\caption{Reward progression under matched GPU-hour budgets.}
\vspace{6pt}
\label{tab:reward-progress}
\begin{tabular}{lcccc}
\toprule
GPU hours & GRPO & SFT-then-RL & LUFFY & TRAPO \\
\midrule
$10 \times 8$ & 0.521 & 0.650 & 0.472 & 0.601 \\
$20 \times 8$ & 0.579 & 0.712 & 0.594 & 0.684 \\
$30 \times 8$ & 0.626 & 0.726 & 0.665 & 0.744 \\
$40 \times 8$ & 0.621 & 0.710 & 0.697 & 0.722 \\
$50 \times 8$ & 0.652 & 0.732 & 0.675 & 0.768 \\
\bottomrule
\end{tabular}
\end{table}

\paragraph{\rev{Summary.}}
\rev{
Across all matched training-time conditions, TRAPO consistently achieves higher reward, which is an indicator of stronger reasoning ability than GRPO and LUFFY, even though it performs fewer training steps. This highlights TRAPO’s advantageous sample efficiency and the effectiveness of combining SFT and RL through a trust-region formulation. Moreover, compared to the SFT-then-RL pipeline, TRAPO attains better performance while requiring less total GPU hours overall.
}

\section{Prefix Usage Dynamics during Training}
\label{appendix:prefix-usage}

\begin{figure}[t]
\begin{center}
\includegraphics[width=0.85\linewidth]{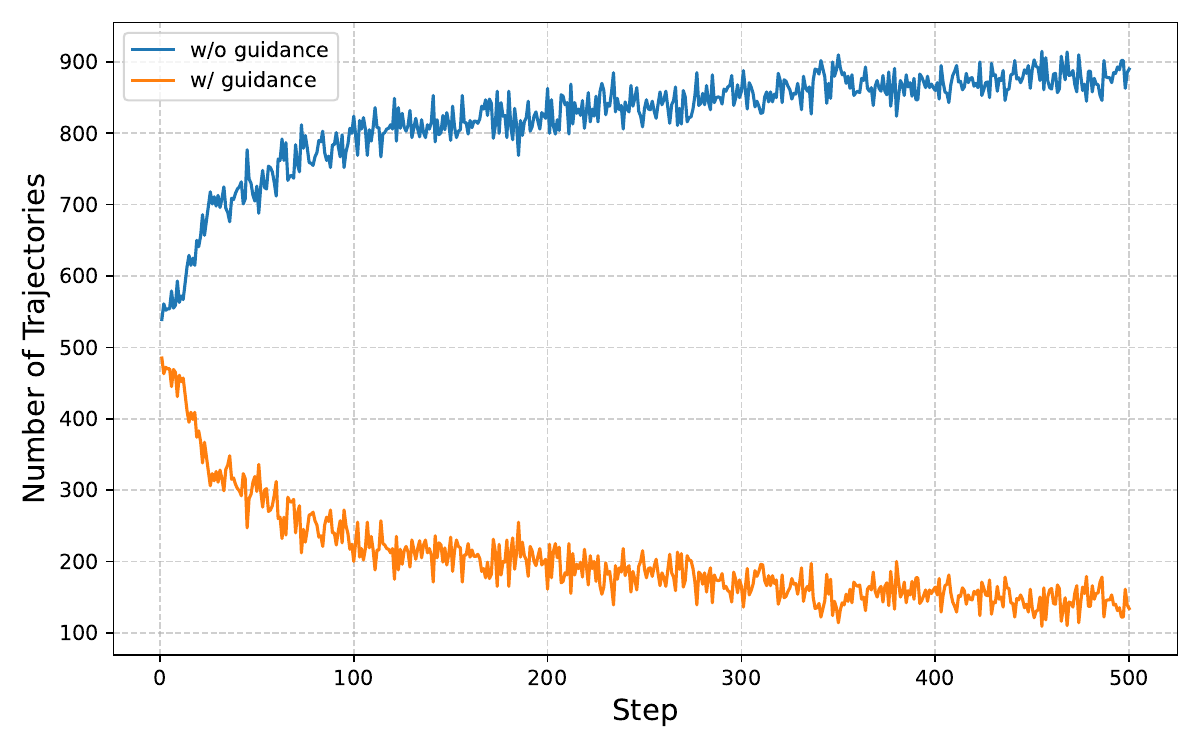}
\end{center}
\vspace{-0.8em}
\caption{Dynamics of prefix usage across training steps. We plot the number of trajectories that receive expert-prefix guidance versus those that remain unguided.}
\label{fig:prefix-num}
\end{figure}

\rev{To better understand how the adaptive prefix-selection mechanism behaves during training, we track the number of trajectories receiving expert prefix guidance at each training step. The results are shown in Figure~\ref{fig:prefix-num}. Two key phenomena emerge:}

\begin{itemize}
    \item \rev{\textbf{Unguided trajectories consistently dominate across the entire training process.}  
    This ensures that the target model maintains strong self-exploration capability and remains aligned with the evaluation setting, which uses no expert prefixes at test time.}

    \item \rev{\textbf{The number of expert-guided trajectories decreases steadily as training progresses.}  
    Early in training, the model frequently falls below the reward thresholds and therefore benefits from partial or full expert-prefix guidance. As performance improves, fewer trajectories trigger guidance, reflecting that the model becomes increasingly capable of solving problems without external intervention.}
\end{itemize}

\rev{This pattern confirms that TRAPO’s adaptive mechanism realizes the intended behavior: it provides expert prefixes only when helpful, and its reliance on expert data naturally diminishes as the target model improves. This demonstrates that TRAPO leverages \emph{smarter} guidance rather than \emph{more} guidance, and avoids over-conditioning on expert trajectories.}

\section{Case Study}
To intuitively demonstrate the differences in style and problem-solving ability of models trained with different methods, we select a polynomial factorization problem from the MATH-500 benchmark as a case study. As shown in Figure~\ref{fig:case_study}, the model trained with GRPO exhibits a consistent forward-reasoning style in problem solving, also retaining the base model’s tendency to employ code verification at the final step. However, due to limited self-reflection and self-correction abilities, it is prone to simple computational and logical mistakes. In contrast, the model trained with SFT demonstrates a rigid imitation of DeepSeek-R1, often engaging in excessive reasoning that exceeds the maximum sequence length budget. The model trained with \mySolution\ not only incorporates the verification and reflective reasoning patterns observed in expert trajectories, but also preserves the base model’s ability to provide concise and clear planning before each reasoning step (e.g., outlining “Step 1 / Step 2 / Step 3 …”). This demonstrates that TRAPO effectively balances the use of expert data with the exploration of its own solution space, combining the strengths of both to maximize training efficiency and reasoning performance.

\begin{figure}[t!]
\begin{center}
\includegraphics[width=\linewidth]{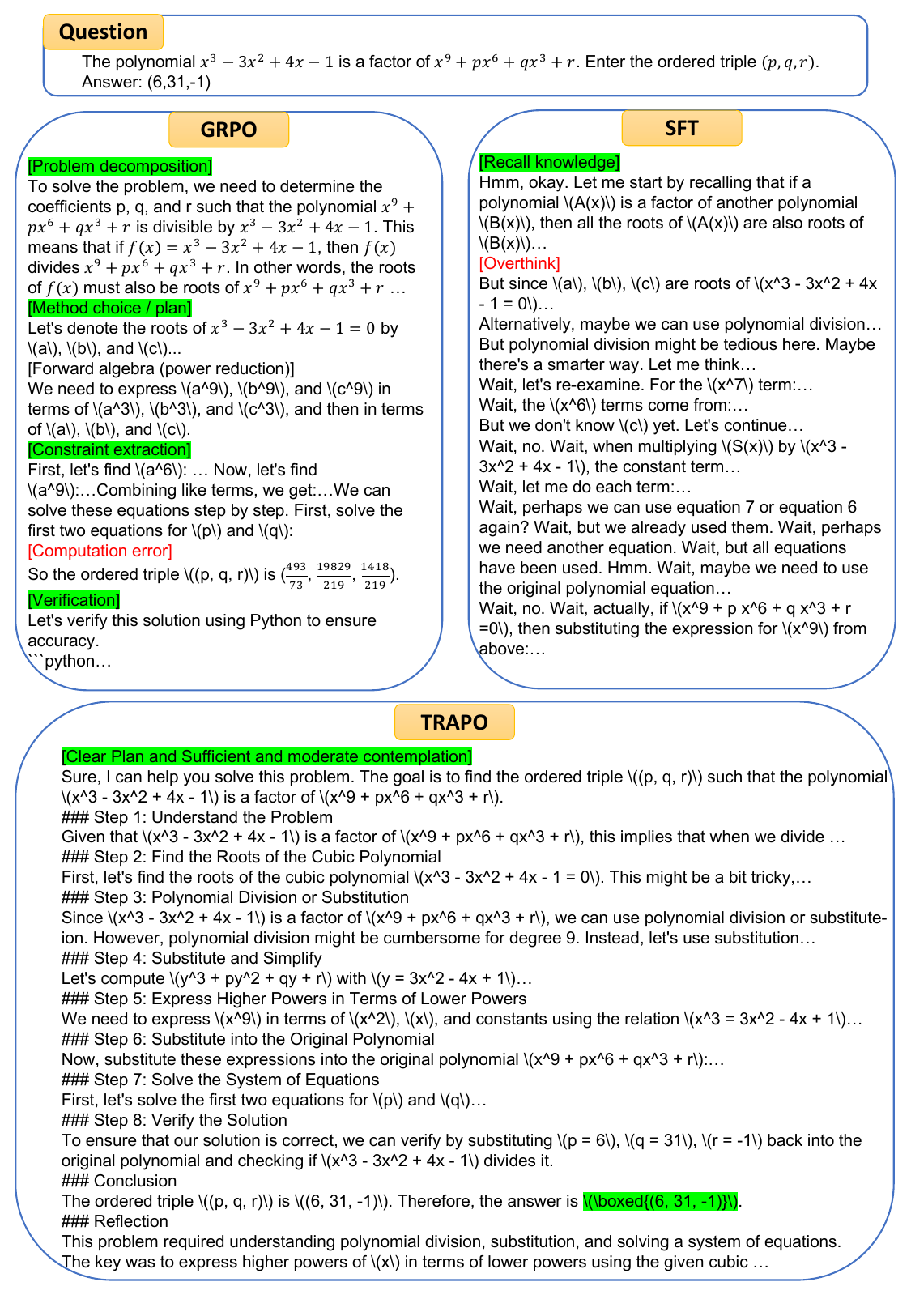}
\end{center}
\vspace{-1em}
\caption{A case study on the MATH-500 benchmark: polynomial factorization problem.}
\label{fig:case_study}
\end{figure}


%% file: iclr2026_conference.bib
@article{openAI-o3,
  title={Open{AI} o3 and o4-mini System Card},
  author={OpenAI},
  journal={https://openai.com/index/o3-o4-mini-system-card/},
  year={2025}
}

@article{guo2025deepseek,
  title={Deepseek-r1: Incentivizing reasoning capability in llms via reinforcement learning},
  author={Guo, Daya and Yang, Dejian and Zhang, Haowei and Song, Junxiao and Zhang, Ruoyu and Xu, Runxin and Zhu, Qihao and Ma, Shirong and Wang, Peiyi and Bi, Xiao and others},
  journal={arXiv preprint arXiv:2501.12948},
  year={2025}
}

@article{team2025kimi,
  title={Kimi k1. 5: Scaling reinforcement learning with llms},
  author={Team, Kimi and Du, Angang and Gao, Bofei and Xing, Bowei and Jiang, Changjiu and Chen, Cheng and Li, Cheng and Xiao, Chenjun and Du, Chenzhuang and Liao, Chonghua and others},
  journal={arXiv preprint arXiv:2501.12599},
  year={2025}
}

@article{chu2025sft,
  title={Sft memorizes, rl generalizes: A comparative study of foundation model post-training},
  author={Chu, Tianzhe and Zhai, Yuexiang and Yang, Jihan and Tong, Shengbang and Xie, Saining and Schuurmans, Dale and Le, Quoc V and Levine, Sergey and Ma, Yi},
  journal={arXiv preprint arXiv:2501.17161},
  year={2025}
}

@article{chen2025sft,
  title={Sft or rl? an early investigation into training r1-like reasoning large vision-language models},
  author={Chen, Hardy and Tu, Haoqin and Wang, Fali and Liu, Hui and Tang, Xianfeng and Du, Xinya and Zhou, Yuyin and Xie, Cihang},
  journal={arXiv preprint arXiv:2504.11468},
  year={2025}
}

@article{minka2005divergence,
  title={Divergence measures and message passing},
  author={Minka, Tom and others},
  year={2005},
  publisher={Technical report, Microsoft Research}
}

@article{malinin2019reverse,
  title={Reverse kl-divergence training of prior networks: Improved uncertainty and adversarial robustness},
  author={Malinin, Andrey and Gales, Mark},
  journal={Advances in neural information processing systems},
  volume={32},
  year={2019}
}

@article{gu2023minillm,
  title={Minillm: Knowledge distillation of large language models},
  author={Gu, Yuxian and Dong, Li and Wei, Furu and Huang, Minlie},
  journal={arXiv preprint arXiv:2306.08543},
  year={2023}
}

@article{shao2024deepseekmath,
  title={Deepseekmath: Pushing the limits of mathematical reasoning in open language models},
  author={Shao, Zhihong and Wang, Peiyi and Zhu, Qihao and Xu, Runxin and Song, Junxiao and Bi, Xiao and Zhang, Haowei and Zhang, Mingchuan and Li, YK and Wu, Yang and others},
  journal={arXiv preprint arXiv:2402.03300},
  year={2024}
}

@inproceedings{agarwal2024policy,
  title={On-policy distillation of language models: Learning from self-generated mistakes},
  author={Agarwal, Rishabh and Vieillard, Nino and Zhou, Yongchao and Stanczyk, Piotr and Garea, Sabela Ramos and Geist, Matthieu and Bachem, Olivier},
  booktitle={The twelfth international conference on learning representations},
  year={2024}
}

@article{lv2025towards,
  title={Towards a Unified View of Large Language Model Post-Training},
  author={Lv, Xingtai and Zuo, Yuxin and Sun, Youbang and Liu, Hongyi and Wei, Yuntian and Chen, Zhekai and He, Lixuan and Zhu, Xuekai and Zhang, Kaiyan and Wang, Bingning and others},
  journal={arXiv preprint arXiv:2509.04419},
  year={2025}
}

@article{huang2025blending,
  title={Blending Supervised and Reinforcement Fine-Tuning with Prefix Sampling},
  author={Huang, Zeyu and Cheng, Tianhao and Qiu, Zihan and Wang, Zili and Xu, Yinghui and Ponti, Edoardo M and Titov, Ivan},
  journal={arXiv preprint arXiv:2507.01679},
  year={2025}
}

@article{liu2025uft,
  title={UFT: Unifying Supervised and Reinforcement Fine-Tuning},
  author={Liu, Mingyang and Farina, Gabriele and Ozdaglar, Asuman},
  journal={arXiv preprint arXiv:2505.16984},
  year={2025}
}

@article{gandhi2025cognitive,
  title={Cognitive behaviors that enable self-improving reasoners, or, four habits of highly effective stars},
  author={Gandhi, Kanishk and Chakravarthy, Ayush and Singh, Anikait and Lile, Nathan and Goodman, Noah D},
  journal={arXiv preprint arXiv:2503.01307},
  year={2025}
}

@article{qwen2.5,
    title   = {Qwen2.5 Technical Report}, 
    author  = {An Yang and Baosong Yang and Beichen Zhang and Binyuan Hui and Bo Zheng and Bowen Yu and Chengyuan Li and Dayiheng Liu and Fei Huang and Haoran Wei and Huan Lin and Jian Yang and Jianhong Tu and Jianwei Zhang and Jianxin Yang and Jiaxi Yang and Jingren Zhou and Junyang Lin and Kai Dang and Keming Lu and Keqin Bao and Kexin Yang and Le Yu and Mei Li and Mingfeng Xue and Pei Zhang and Qin Zhu and Rui Men and Runji Lin and Tianhao Li and Tingyu Xia and Xingzhang Ren and Xuancheng Ren and Yang Fan and Yang Su and Yichang Zhang and Yu Wan and Yuqiong Liu and Zeyu Cui and Zhenru Zhang and Zihan Qiu},
    journal = {arXiv preprint arXiv:2412.15115},
    year    = {2024}
}

@article{yang2024qwen2,
  title={Qwen2. 5-math technical report: Toward mathematical expert model via self-improvement},
  author={Yang, An and Zhang, Beichen and Hui, Binyuan and Gao, Bofei and Yu, Bowen and Li, Chengpeng and Liu, Dayiheng and Tu, Jianhong and Zhou, Jingren and Lin, Junyang and others},
  journal={arXiv preprint arXiv:2409.12122},
  year={2024}
}

@article{ghojogh2021kkt,
  title={KKT conditions, first-order and second-order optimization, and distributed optimization: tutorial and survey},
  author={Ghojogh, Benyamin and Ghodsi, Ali and Karray, Fakhri and Crowley, Mark},
  journal={arXiv preprint arXiv:2110.01858},
  year={2021}
}

@inproceedings{torabi2018behavioral,
  title={Behavioral cloning from observation},
  author={Torabi, Faraz and Warnell, Garrett and Stone, Peter},
  booktitle={Proceedings of the 27th International Joint Conference on Artificial Intelligence},
  pages={4950--4957},
  year={2018}
}

@article{lightman2023let,
  title={Let's verify step by step},
  author={Lightman, Hunter and Kosaraju, Vineet and Burda, Yura and Edwards, Harri and Baker, Bowen and Lee, Teddy and Leike, Jan and Schulman, John and Sutskever, Ilya and Cobbe, Karl},
  journal={arXiv preprint arXiv:2305.20050},
  year={2023}
}

@misc{chen2021evaluatinglargelanguagemodels,
      title={Evaluating Large Language Models Trained on Code}, 
      author={Mark Chen and Jerry Tworek and Heewoo Jun and Qiming Yuan and Henrique Ponde de Oliveira Pinto and Jared Kaplan and Harri Edwards and Yuri Burda and Nicholas Joseph and Greg Brockman and Alex Ray and Raul Puri and Gretchen Krueger and Michael Petrov and Heidy Khlaaf and Girish Sastry and Pamela Mishkin and Brooke Chan and Scott Gray and Nick Ryder and Mikhail Pavlov and Alethea Power and Lukasz Kaiser and Mohammad Bavarian and Clemens Winter and Philippe Tillet and Felipe Petroski Such and Dave Cummings and Matthias Plappert and Fotios Chantzis and Elizabeth Barnes and Ariel Herbert-Voss and William Hebgen Guss and Alex Nichol and Alex Paino and Nikolas Tezak and Jie Tang and Igor Babuschkin and Suchir Balaji and Shantanu Jain and William Saunders and Christopher Hesse and Andrew N. Carr and Jan Leike and Josh Achiam and Vedant Misra and Evan Morikawa and Alec Radford and Matthew Knight and Miles Brundage and Mira Murati and Katie Mayer and Peter Welinder and Bob McGrew and Dario Amodei and Sam McCandlish and Ilya Sutskever and Wojciech Zaremba},
      year={2021},
      eprint={2107.03374},
      archivePrefix={arXiv},
      primaryClass={cs.LG},
      url={https://arxiv.org/abs/2107.03374}, 
}

@inproceedings{
shridhar2021alfworld,
title={{\{}ALFW{\}}orld: Aligning Text and Embodied Environments for Interactive Learning},
author={Mohit Shridhar and Xingdi Yuan and Marc-Alexandre Cote and Yonatan Bisk and Adam Trischler and Matthew Hausknecht},
booktitle={International Conference on Learning Representations},
year={2021},
url={https://openreview.net/forum?id=0IOX0YcCdTn}
}

@article{yan2025learning,
  title={Learning to reason under off-policy guidance},
  author={Yan, Jianhao and Li, Yafu and Hu, Zican and Wang, Zhi and Cui, Ganqu and Qu, Xiaoye and Cheng, Yu and Zhang, Yue},
  journal={arXiv preprint arXiv:2504.14945},
  year={2025}
}

@misc{openr1,
    title = {Open R1: A fully open reproduction of DeepSeek-R1},
    url = {https://github.com/huggingface/open-r1},
    author = {Hugging Face},
    month = {January},
    year = {2025}
}

@article{fu2025srft,
  title={SRFT: A Single-Stage Method with Supervised and Reinforcement Fine-Tuning for Reasoning},
  author={Fu, Yuqian and Chen, Tinghong and Chai, Jiajun and Wang, Xihuai and Tu, Songjun and Yin, Guojun and Lin, Wei and Zhang, Qichao and Zhu, Yuanheng and Zhao, Dongbin},
  journal={arXiv preprint arXiv:2506.19767},
  year={2025}
}

@misc{li2024numinamath,
  author       = {Jia Li and Edward Beeching and Lewis Tunstall and Ben Lipkin and Roman Soletskyi and Shengyi Huang and Kashif Rasul and Longhui Yu and Albert Q. Jiang and Ziju Shen and others},
  title        = {Numinamath: The largest public dataset in AI4Maths with 860k pairs of competition math problems and solutions},
  year         = {2024},
  howpublished = {\url{https://huggingface.co/datasets/Numinamath}},
  note         = {Hugging Face repository, 13:9}
}

@inproceedings{DBLP:conf/acl/HeLBHTSHHHZLQL024,
  author       = {Chaoqun He and
                  Renjie Luo and
                  Yuzhuo Bai and
                  Shengding Hu and
                  Zhen Leng Thai and
                  Junhao Shen and
                  Jinyi Hu and
                  Xu Han and
                  Yujie Huang and
                  Yuxiang Zhang and
                  Jie Liu and
                  Lei Qi and
                  Zhiyuan Liu and
                  Maosong Sun},
  title        = {OlympiadBench: {A} Challenging Benchmark for Promoting {AGI} with
                  Olympiad-Level Bilingual Multimodal Scientific Problems},
  booktitle    = {Proceedings of the 62nd Annual Meeting of the Association for Computational
                  Linguistics (Volume 1: Long Papers), {ACL} 2024, Bangkok, Thailand,
                  August 11-16, 2024},
  pages        = {3828--3850},
  publisher    = {Association for Computational Linguistics},
  year         = {2024},
  url          = {https://doi.org/10.18653/v1/2024.acl-long.211},
}

@inproceedings{DBLP:conf/nips/LewkowyczADDMRS22,
  author       = {Aitor Lewkowycz and
                  Anders Andreassen and
                  David Dohan and
                  Ethan Dyer and
                  Henryk Michalewski and
                  Vinay V. Ramasesh and
                  Ambrose Slone and
                  Cem Anil and
                  Imanol Schlag and
                  Theo Gutman{-}Solo and
                  Yuhuai Wu and
                  Behnam Neyshabur and
                  Guy Gur{-}Ari and
                  Vedant Misra},
  title        = {Solving Quantitative Reasoning Problems with Language Models},
  booktitle    = {Advances in Neural Information Processing Systems 35: Annual Conference
                  on Neural Information Processing Systems 2022, NeurIPS 2022, New Orleans,
                  LA, USA, November 28 - December 9, 2022},
  year         = {2022},
  url          = {http://papers.nips.cc/paper\_files/paper/2022/hash/18abbeef8cfe9203fdf9053c9c4fe191-Abstract-Conference.html},
}

@inproceedings{DBLP:conf/nips/HendrycksBKABTS21,
  author       = {Dan Hendrycks and
                  Collin Burns and
                  Saurav Kadavath and
                  Akul Arora and
                  Steven Basart and
                  Eric Tang and
                  Dawn Song and
                  Jacob Steinhardt},
  title        = {Measuring Mathematical Problem Solving With the {MATH} Dataset},
  booktitle    = {Proceedings of the Neural Information Processing Systems Track on
                  Datasets and Benchmarks 1, NeurIPS Datasets and Benchmarks 2021, December
                  2021, virtual},
  year         = {2021},
  url          = {https://datasets-benchmarks-proceedings.neurips.cc/paper/2021/hash/be83ab3ecd0db773eb2dc1b0a17836a1-Abstract-round2.html},
}

@article{DBLP:journals/corr/abs-1803-05457,
  author       = {Peter Clark and
                  Isaac Cowhey and
                  Oren Etzioni and
                  Tushar Khot and
                  Ashish Sabharwal and
                  Carissa Schoenick and
                  Oyvind Tafjord},
  title        = {Think you have Solved Question Answering? Try ARC, the {AI2} Reasoning
                  Challenge},
  journal      = {CoRR},
  volume       = {abs/1803.05457},
  year         = {2018},
  url          = {http://arxiv.org/abs/1803.05457},
}

@inproceedings{DBLP:conf/nips/WangMZNCGRAHJLK24,
  author       = {Yubo Wang and
                  Xueguang Ma and
                  Ge Zhang and
                  Yuansheng Ni and
                  Abhranil Chandra and
                  Shiguang Guo and
                  Weiming Ren and
                  Aaran Arulraj and
                  Xuan He and
                  Ziyan Jiang and
                  Tianle Li and
                  Max Ku and
                  Kai Wang and
                  Alex Zhuang and
                  Rongqi Fan and
                  Xiang Yue and
                  Wenhu Chen},
  title        = {MMLU-Pro: {A} More Robust and Challenging Multi-Task Language Understanding
                  Benchmark},
  booktitle    = {Advances in Neural Information Processing Systems 38: Annual Conference
                  on Neural Information Processing Systems 2024, NeurIPS 2024, Vancouver,
                  BC, Canada, December 10 - 15, 2024},
  year         = {2024},
  url          = {http://papers.nips.cc/paper\_files/paper/2024/hash/ad236edc564f3e3156e1b2feafb99a24-Abstract-Datasets\_and\_Benchmarks\_Track.html},
}

@article{liu2025understanding,
  title={Understanding r1-zero-like training: A critical perspective},
  author={Liu, Zichen and Chen, Changyu and Li, Wenjun and Qi, Penghui and Pang, Tianyu and Du, Chao and Lee, Wee Sun and Lin, Min},
  journal={arXiv preprint arXiv:2503.20783},
  year={2025}
}

@article{cui2025process,
  title={Process reinforcement through implicit rewards},
  author={Cui, Ganqu and Yuan, Lifan and Wang, Zefan and Wang, Hanbin and Li, Wendi and He, Bingxiang and Fan, Yuchen and Yu, Tianyu and Xu, Qixin and Chen, Weize and others},
  journal={arXiv preprint arXiv:2502.01456},
  year={2025}
}

@article{zeng2025simplerl,
  title={Simplerl-zoo: Investigating and taming zero reinforcement learning for open base models in the wild},
  author={Zeng, Weihao and Huang, Yuzhen and Liu, Qian and Liu, Wei and He, Keqing and Ma, Zejun and He, Junxian},
  journal={arXiv preprint arXiv:2503.18892},
  year={2025}
}

@article{hu2025open,
  title={Open-reasoner-zero: An open source approach to scaling up reinforcement learning on the base model},
  author={Hu, Jingcheng and Zhang, Yinmin and Han, Qi and Jiang, Daxin and Zhang, Xiangyu and Shum, Heung-Yeung},
  journal={arXiv preprint arXiv:2503.24290},
  year={2025}
}

@article{ma2025learning,
  title={Learning What Reinforcement Learning Can't: Interleaved Online Fine-Tuning for Hardest Questions},
  author={Ma, Lu and Liang, Hao and Qiang, Meiyi and Tang, Lexiang and Ma, Xiaochen and Wong, Zhen Hao and Niu, Junbo and Shen, Chengyu and He, Runming and Cui, Bin and others},
  journal={arXiv preprint arXiv:2506.07527},
  year={2025}
}

@article{snell2024scaling,
  title={Scaling llm test-time compute optimally can be more effective than scaling model parameters},
  author={Snell, Charlie and Lee, Jaehoon and Xu, Kelvin and Kumar, Aviral},
  journal={arXiv preprint arXiv:2408.03314},
  year={2024}
}

@article{yue2025does,
  title={Does reinforcement learning really incentivize reasoning capacity in llms beyond the base model?},
  author={Yue, Yang and Chen, Zhiqi and Lu, Rui and Zhao, Andrew and Wang, Zhaokai and Song, Shiji and Huang, Gao},
  journal={arXiv preprint arXiv:2504.13837},
  year={2025}
}

@article{wang2022self,
  title={Self-consistency improves chain of thought reasoning in language models},
  author={Wang, Xuezhi and Wei, Jason and Schuurmans, Dale and Le, Quoc and Chi, Ed and Narang, Sharan and Chowdhery, Aakanksha and Zhou, Denny},
  journal={arXiv preprint arXiv:2203.11171},
  year={2022}
}

@article{jaech2024openai,
  title={Openai o1 system card},
  author={Jaech, Aaron and Kalai, Adam and Lerer, Adam and Richardson, Adam and El-Kishky, Ahmed and Low, Aiden and Helyar, Alec and Madry, Aleksander and Beutel, Alex and Carney, Alex and others},
  journal={arXiv preprint arXiv:2412.16720},
  year={2024}
}

@article{xi2024training,
  title={Training large language models for reasoning through reverse curriculum reinforcement learning},
  author={Xi, Zhiheng and Chen, Wenxiang and Hong, Boyang and Jin, Senjie and Zheng, Rui and He, Wei and Ding, Yiwen and Liu, Shichun and Guo, Xin and Wang, Junzhe and others},
  journal={arXiv preprint arXiv:2402.05808},
  year={2024}
}

@article{shi2025efficient,
  title={Efficient reinforcement finetuning via adaptive curriculum learning},
  author={Shi, Taiwei and Wu, Yiyang and Song, Linxin and Zhou, Tianyi and Zhao, Jieyu},
  journal={arXiv preprint arXiv:2504.05520},
  year={2025}
}

@article{xie2025logic,
  title={Logic-rl: Unleashing llm reasoning with rule-based reinforcement learning},
  author={Xie, Tian and Gao, Zitian and Ren, Qingnan and Luo, Haoming and Hong, Yuqian and Dai, Bryan and Zhou, Joey and Qiu, Kai and Wu, Zhirong and Luo, Chong},
  journal={arXiv preprint arXiv:2502.14768},
  year={2025}
}

@article{yu2025dapo,
  title={Dapo: An open-source llm reinforcement learning system at scale},
  author={Yu, Qiying and Zhang, Zheng and Zhu, Ruofei and Yuan, Yufeng and Zuo, Xiaochen and Yue, Yu and Dai, Weinan and Fan, Tiantian and Liu, Gaohong and Liu, Lingjun and others},
  journal={arXiv preprint arXiv:2503.14476},
  year={2025}
}

@article{yue2025vapo,
  title={Vapo: Efficient and reliable reinforcement learning for advanced reasoning tasks},
  author={Yue, Yu and Yuan, Yufeng and Yu, Qiying and Zuo, Xiaochen and Zhu, Ruofei and Xu, Wenyuan and Chen, Jiaze and Wang, Chengyi and Fan, TianTian and Du, Zhengyin and others},
  journal={arXiv preprint arXiv:2504.05118},
  year={2025}
}

@article{he2025amft,
  title={AMFT: Aligning LLM Reasoners by Meta-Learning the Optimal Imitation-Exploration Balance},
  author={He, Lixuan and Feng, Jie and Li, Yong},
  journal={arXiv preprint arXiv:2508.06944},
  year={2025}
}

@article{hu2024openrlhf,
  title={Openrlhf: An easy-to-use, scalable and high-performance rlhf framework},
  author={Hu, Jian and Wu, Xibin and Shen, Wei and Liu, Jason Klein and Zhu, Zilin and Wang, Weixun and Jiang, Songlin and Wang, Haoran and Chen, Hao and Chen, Bin and others},
  journal={arXiv preprint arXiv:2405.11143},
  year={2024}
}

@inproceedings{kwon2023efficient,
  title={Efficient memory management for large language model serving with pagedattention},
  author={Kwon, Woosuk and Li, Zhuohan and Zhuang, Siyuan and Sheng, Ying and Zheng, Lianmin and Yu, Cody Hao and Gonzalez, Joseph and Zhang, Hao and Stoica, Ion},
  booktitle={Proceedings of the 29th symposium on operating systems principles},
  pages={611--626},
  year={2023}
}

@inproceedings{sheng2025hybridflow,
  title={Hybridflow: A flexible and efficient rlhf framework},
  author={Sheng, Guangming and Zhang, Chi and Ye, Zilingfeng and Wu, Xibin and Zhang, Wang and Zhang, Ru and Peng, Yanghua and Lin, Haibin and Wu, Chuan},
  booktitle={Proceedings of the Twentieth European Conference on Computer Systems},
  pages={1279--1297},
  year={2025}
}

@misc{liu2025oat,
  title={Oat: A research-friendly framework for llm online alignment},
  author={Liu, Zichen and Chen, Changyu and Du, Chao and Lee, Wee Sun and Lin, Min},
  year={2025}
}
